\documentclass[twoside,11pt]{article}

\usepackage{blindtext}

\usepackage[preprint]{jmlr2e}

\usepackage{lastpage}

\ShortHeadings{Coordinate Descent for Learning Bayesian Networks}{Xu, Küçükyavuz, Shojaie, and Taeb}
\firstpageno{1}

\usepackage{lipsum}
\usepackage{amsfonts}
\usepackage{graphicx}
\usepackage{epstopdf}
\usepackage{algorithm, algorithmic}
\ifpdf
  \DeclareGraphicsExtensions{.eps,.pdf,.png,.jpg}
\else
  \DeclareGraphicsExtensions{.eps}
\fi

\usepackage{mathtools}

\usepackage{subcaption}

\usepackage{tkz-graph}
\usetikzlibrary{graphs}
\usepackage{tikz-cd} 
\usepackage[export]{adjustbox}

\usepackage{amsmath}

\DeclareMathOperator*{\argmin}{arg\,min}

\newcommand{\R}[0]{\mathbb{R}}

\newcommand{\norm}[1]{\|#1\|}

\def\T{{ \mathrm{\scriptscriptstyle T} }}

\def\inv{{ \mathrm{\scriptscriptstyle -1} }}
\def\invT{{ \mathrm{\scriptscriptstyle -T} }}

\definecolor{OliveGreen}{rgb}{0,0.6,0}

\newtheorem{assumption}{Assumption}

\def\rgap{\textsc{RGAP}}

\newcommand{\ignore}[1]{}

\begin{document}

\title{An Asymptotically Optimal Coordinate Descent Algorithm for Learning Bayesian Networks from Gaussian Models}

\author{\name Tong Xu \email tongxu2027@u.northwestern.edu \\
       \addr Department of Industrial Engineering and Management Sciences\\
       Northwestern University\\
       Evanston, IL 60208, USA
       \AND
       \name Simge Küçükyavuz \email simge@northwestern.edu \\
       \addr Department of Industrial Engineering and Management Sciences\\
       Northwestern University\\
       Evanston, IL 60208, USA
       \AND Ali Shojaie \email ashojaie@uw.edu \\
       \addr Department of Biostatistics\\
       University of Washington\\
       Seattle, WA 98195-1617, USA
       \AND Armeen Taeb \email ataeb@uw.edu \\
       \addr Department of Statistics\\
       University of Washington\\
       Seattle, WA 98195-4322, USA}

\editor{}

\maketitle

\begin{abstract}%
This paper studies the problem of learning Bayesian networks from continuous observational data, generated according to a linear Gaussian structural equation model. We consider an $\ell_0$-penalized maximum likelihood estimator for this problem, which is known to have favorable statistical properties but is computationally challenging to solve, especially for medium-sized Bayesian networks.  
We propose a new coordinate descent algorithm to approximate this estimator and prove several remarkable properties of our procedure: 
The algorithm converges to a coordinate-wise minimum, and despite the non-convexity of the loss function, as the sample size tends to infinity, the objective value of the coordinate descent solution converges to the optimal objective value of the $\ell_0$-penalized maximum likelihood estimator.  To the best of our knowledge, our proposal is the first coordinate descent procedure endowed with optimality guarantees in the context of learning Bayesian networks. Numerical experiments on synthetic and real data demonstrate that our coordinate descent method can obtain near-optimal solutions while being scalable.
\end{abstract}


\section{Introduction}
Bayesian networks provide a powerful framework for modeling causal relationships among a collection of random variables. A Bayesian network is typically represented by a directed acyclic graph (DAG), where the random variables are encoded as vertices (or nodes), a directed edge from node $i$ to node $j$ indicates that $i$ causes $j$, and the acyclic property of the graph prevents the occurrence of circular dependencies. If the DAG is known, it can be used to predict the behavior of the system under manipulations or interventions. However, in large systems such as gene regulatory networks, the DAG is not known a priori, making it necessary to develop efficient and rigorous methods to learn the graph from data. To solve this problem using only observational data, we assume that all relevant variables are observed and that we only have access to observational data.

Three broad classes of methods for learning DAGs from data are constraint-based, score-based, and hybrid. Constraint-based methods use repeated conditional independence tests to determine the presence of edges in a DAG. A prominent example is the PC algorithm and its extensions  \citep{causalitybase, Tsamardinos2006TheMH}. While the PC algorithm can be applied in non-parametric settings, testing for conditional independencies is generally hard \citep{Shah2018TheHO}. Furthermore, even in the Gaussian setting, statistical consistency guarantees for the PC algorithm are shown under the \emph{strong faithfulness} condition \citep{kalisch2007estimating}, which is known to be restrictive in high-dimensional settings \citep{Uhler2012GeometryOT}. Score-based methods often deploy a penalized log-likelihood as a score function and search over the space of DAGs to identify a DAG with an optimal score. These approaches do not require the strong faithfulness assumption. However, statistical guarantees are not provided for many score-based approaches and solving them exactly suffers from high computational complexity. For example, learning an optimal graph using dynamic programming takes about 10 hours for a medium-size problem with 29 nodes \citep{silander2012simple}. Several papers \citep{kucukyavuz2022consistent,xu2024integer} offer speedup by casting the problem as a convex mixed-integer program, but finding an optimal solution with these approaches can still take an hour for a medium-sized problem. Finally, hybrid approaches combine constraint-based and score-based methods by using background knowledge or conditional independence tests to restrict the DAG search space \citep{Tsamardinos2006TheMH,nandy2018high}.

Several strategies have been developed to make score-based methods more scalable by finding approximate solutions instead of finding optimally scoring DAGs. One direction to find good approximate solutions is to resort to greedy-based methods, with a prominent example being the Greedy Equivalence Search (GES) algorithm \citep{chickering2002optimal}. GES performs a greedy search on the space of completed partially directed acyclic graphs (an equivalence class of DAGs) and is known to produce asymptotically consistent solutions \citep{chickering2002optimal}. Despite its favorable properties, GES does not provide optimality or consistency guarantees for any finite sample size. Further, the guarantees of GES assume a fixed number of nodes with sample size going to infinity and do not allow for a growing number of nodes. 
Another direction is gradient-based approaches \citep{yu2019dag,zheng2018dags}, which relax the discrete search space over DAGs to a continuous search space, allowing gradient descent and other techniques from continuous optimization to be applied. However, the search space for these problems is highly non-convex, resulting in limited guarantees for convergence, even to a local minimum. Finally, another notable direction is based on   \emph{coordinate descent}; that is iteratively maximizing the given score function over a single parameter, while keeping the remaining parameters fixed and checking that the resulting model is a DAG at each update \citep{aragam2015concave, aragam2019learning, Fu13, ye2020optimizing}. While coordinate descent algorithms have shown significant promise in learning large-scale Bayesian networks, to the best of our knowledge, they do not come with convergence and optimality guarantees.

\textbf{Contributions:} We propose a new score-based coordinate descent algorithm for learning Bayesian networks from Gaussian linear structural equation models. Remarkably, unlike prior coordinate descent algorithms for learning Bayesian networks, our procedure provably i) converges to a coordinate-wise minimum, and ii) yields optimally scoring DAGs in the large sample limit despite the non-convex nature of the problem. Moreover, we characterize the convergence rate as a function of both the sample size and the number of nodes. A byproduct of our analysis is a characterization of the optimization landscape: in the large-sample regime, all local minima achieve objective values close a global minimum, and exactly match it in the limit---a result that may be of independent interest.

As a scoring function for this approach, we deploy an $\ell_0$-penalized Gaussian log-likelihood, which implies that optimally-scoring DAGs are solutions to a highly non-convex $\ell_0$-penalized maximum likelihood estimator. This estimator is known to have strong statistical consistency guarantees \citep{vgBuhlmann}, but solving it is, in general, intractable. Thus, our coordinate descent algorithm can be viewed as a scalable and efficient approach to finding approximate solutions to this estimator that are asymptotically optimal (that is, match the optimal objective value of the $\ell_0$ penalized maximum-likelihood estimator as the sample size tends to infinity).

We illustrate the advantages of our method over competing approaches via extensive numerical experiments. The proposed approach is implemented in the python package \emph{micodag}, and all numerical results and figures can be reproduced using the code in \url{https://github.com/AtomXT/coordinate-descent-for-bayesian-networks.git}.

\section{Problem Setup}
\label{sec:formulation}

Consider an unknown DAG whose $m$ nodes correspond to observed random variables $X \in \mathbb{R}^m$. We denote the DAG by $\mathcal{G}^\star = (V,E^\star)$ where $V=\left\{1, \ldots, m\right\}$ is the vertex set and $E^\star \subseteq V \times V$ is the directed edge set, {where $(i,j) \in E^\star$ indicates $i \to j$ in graph $\mathcal{G}^\star$}. We assume that the random variables $X$ satisfy the linear structural equation model (SEM): 
\begin{equation}
    \label{eqn:sem}
    X = {B^\star}^\T X + \epsilon,
\end{equation}
where $B^{\star}\in \R^{m\times m}$ is the connectivity matrix with zeros on the diagonal and  $B^\star_{jk} \neq 0$ if $(j,k) \in E^\star$. In other words, the sparsity pattern of $B^\star$ encodes the true DAG structure. Further, $\epsilon \sim \mathcal{N}(0, \Omega^{\star})$ is a random Gaussian noise vector with zero mean and independent coordinates so that $\Omega^\star$ is a diagonal matrix. We use the Bachman-Landau symbols $\mathcal{O}$ to describe the limiting behavior of a function. Furthermore, we denote $z \asymp 1$ to express $z = \mathcal{O}(1)$ and $1/z = \mathcal{O}(1)$. Assuming, without loss of generality, that all random variables are centered, each variable $X_j$ in this model can be expressed as the
linear combination of its parents---the set of nodes with directed edges pointing to $j$---plus independent Gaussian noise. By the SEM \eqref{eqn:sem} and the Gaussianity of $\epsilon$, the random vector $X$ follows the Gaussian distribution $\mathcal{P}^{\star} = \mathcal{N}(0, \Sigma^\star)$, with $\Sigma^\star = (I-B^\star)^\invT\Omega^\star(I-B^\star)^\inv$. Throughout, we assume that the distribution $\mathcal{P}^\star$ is non-degenerate, or equivalently, $\Sigma^\star$ is positive definite. Our objective is to estimate the matrix $B^\star$, or as we describe next, an equivalence class when the underlying model is not identifiable.

Multiple SEMs are generally compatible with the distribution $\mathcal{P}^\star$. To formalize this, we need the following definition. 
\begin{definition}(Graph $\mathcal{G}(B)$ induced by $B$) Let $B \in \mathbb{R}^{m \times m}$ with zeros on the diagonal. Then, $\mathcal{G}(B)$ is the directed graph on $m$ nodes where the directed edge from $i$ to $j$ appears in $\mathcal{G}(B)$ if and only if $B_{ij} \neq 0$.
\end{definition} 
To see why the model \eqref{eqn:sem} is generally not identifiable, note that there are multiple tuples $(B,\Omega)$ where $\mathcal{G}(B)$ is DAG and $\Omega$ is a positive definite diagonal matrix with $\Sigma^\star = (I-B)^\invT\Omega(I-B)^\inv$ \citep{vgBuhlmann}. As a result, the SEM given by $(B,\Omega)$ yields an equally representative model as the one given by the population parameters $(B^\star,\Omega^\star)$. When $\mathcal{G}^\star$ is \emph{faithful} with respect to the graph $\mathcal{G}^\star$, the sparsest DAGs that are compatible with $\mathcal P^\star$ are precisely $\mathrm{MEC}(\mathcal{G}^\star)$, the \emph{Markov equivalence class} of $\mathcal{G}^\star$ \citep{vgBuhlmann}. Next, we formally define the Markov equivalence class. 

\begin{definition}(Markov equivalence class $\mathrm{MEC}(\mathcal{G})$\citep{verma1990equivalence}) Let $\mathcal{G} = (V,E)$ be a DAG. Then,  $\mathrm{MEC}(\mathcal{G})$ consists of DAGs that have the same skeleton and same v-structures as $\mathcal{G}$. The skeleton of $\mathcal{G}$ is the undirected graph obtained from $\mathcal{G}$ by substituting directed edges with undirected ones. Furthermore, nodes $i,j$, and $k$ form a v-structure if $(i,k) \in E$ and $(j,k) \in E$, and there is no edge between $i$ and $j$.
\end{definition}

With these preliminaries and definitions in place, we next describe the proposed \texorpdfstring{{$\ell_0$}}{l0}-penalized maximum likelihood estimator. 
Consider $n$ independent and identically distributed observations of the random vector $X$ generated according to \eqref{eqn:sem}. Let $\hat{\Sigma}$ be the sample covariance matrix obtained from these observations. Further, consider a Gaussian SEM parameterized by connectivity matrix $B$ and noise variance $\Omega$ with $D = \Omega^{-1}$. The parameters $(B,D)$ specify the following precision, or inverse covariance, matrix $\Theta := \Theta(B,D) := (I-B){D}(I-B)^\T$. The negative log-likelihood of this SEM is proportional to $\ell_n(\Theta) = \mathrm{trace}(\Theta\hat{\Sigma})-\log\det(\Theta)$.
Naturally, we seek a model that not only has a small negative log-likelihood but is also specified by a sparse connectivity matrix containing few nonzero elements. Thus, we deploy the following $\ell_0$-penalized maximum likelihood estimator with a regularization parameter $\lambda \geq 0$:
\begin{align}\label{Problem:original}
    \min_{B \in \mathbb{R}^{m \times m},D \in \mathbb{D}^m_{++}} \ell_n\left(\left(I-B\right) D \left(I-B\right)^\T\right) + \lambda^2\norm{B}_{\ell_0}\quad \text {s.t.} \quad\mathcal{G}(B) \text{ is a DAG}.
\end{align}
Here, $\mathbb{D}^m_{++}$ denotes the collection of positive definite $m \times m$ diagonal matrices and $\|B\|_{\ell_0}$ denotes the number of non-zeros in $B$. Note that the $\ell_0$ penalty is generally preferred over the $\ell_1$ penalty or minimax concave penalty (MCP) for penalizing the complexity of the model. In particular, $\ell_0$ regularization exhibits the important property that equivalent DAGs---those in the same Markov equivalence class---have the same penalized likelihood score, while this is not the case for $\ell_1$ or MCP regularization \citep{vgBuhlmann}. Indeed, this lack of score invariance with $\ell_1$ regularization partially explains the unfavorable properties of some existing methods (see Section~\ref{sec:experiments}).

The Markov equivalence class $\mathrm{MEC}(\mathcal{G}(\hat{B}^\mathrm{opt}))$ of the connectivity matrix $\hat{B}^\mathrm{opt}$ obtained from solving \eqref{Problem:original} provides an estimate of $\mathrm{MEC}(\mathcal{G}^\star)$. \citet{vgBuhlmann} prove that this estimate has desirable statistical properties; however, solving it is, in general, intractable. As stated, the objective function $\ell_n((I-B) D (I-B)^\T)$ is non-convex and non-linear function of $(B, D)$. Furthermore, the $\log\det$ function in the likelihood $\ell_n$ is not amenable to standard mixed-integer programming optimization techniques. To circumvent the aforementioned challenges, \citet{xu2024integer} derive the following equivalent optimization model via the change of variables $\Gamma \leftarrow (I-B)D^{1/2}$:
\begin{equation}
    \min\limits_{\Gamma \in \mathbb{R}^{m \times m}}  f(\Gamma)\quad 
\text{s.t.} \quad \mathcal{G}\left(\Gamma-\mathrm{diag}\left(\Gamma\right)\right) \text{ is a DAG}.
\label{Problem:micp}
\end{equation}
Here $f(\Gamma):= \sum_{i=1}^m-2\log(\Gamma_{ii})+\mathrm{tr}(\Gamma\Gamma^\T\hat{\Sigma}) + \lambda^2\norm{\Gamma-\mathrm{diag}(\Gamma)}_{\ell_0},$
and $\mathrm{diag}(\Gamma)$ is the diagonal matrix formed by taking the diagonal entries of $\Gamma$. The optimal solutions of \eqref{Problem:original} and \eqref{Problem:micp} are directly connected: Letting $(\hat{B}^\mathrm{opt},\hat{D}^\mathrm{opt})$ be an optimal solution of \eqref{Problem:original}, then $\hat{\Gamma}^\mathrm{opt} = (I-\hat{B}^\mathrm{opt})(\hat{D}^\mathrm{opt})^{1/2}$ is an optimal solution of \eqref{Problem:micp}. Furthermore, the sparsity pattern of $\hat{\Gamma}^\mathrm{opt}-\mathrm{diag}(\hat{\Gamma}^\mathrm{opt})$ is the same as that of $\hat{B}^\mathrm{opt}$; in other words, the Markov equivalence class $\mathrm{MEC}(\mathcal{G}(\hat{B}^\mathrm{opt}))$ is the same as the Markov equivalence class $\mathrm{MEC}(\mathcal{G}(\hat{\Gamma}^\mathrm{opt}-\mathrm{diag}(\hat{\Gamma}^\mathrm{opt})))$.

\citet{xu2024integer} recast the optimization problem \eqref{Problem:micp} as a convex mixed-integer program and provide algorithms to solve \eqref{Problem:micp} to optimality. However, solving \eqref{Problem:micp} is, in general, NP-hard, and obtaining optimality certificates may take an hour for a problem with 20 nodes \citep{xu2024integer}. 

\section{A Coordinate Descent Algorithm for DAG Learning}
In this section, we develop a cyclic coordinate descent approach to find a heuristic solution to problem \eqref{Problem:micp}. The coordinate descent solver is fast and can be scaled to large-scale problems. As we demonstrate in Section~\ref{sec:theory}, it provably converges and produces an asymptotically optimal solution to \eqref{Problem:micp}. Given the quality of its estimates, the proposed coordinate descent algorithm can also be used as a warm start for the mixed-integer programming framework in \citet{xu2024integer} to obtain optimal solutions.

\subsection{Parameter update without acyclicity constraints}
\label{sec:parameter_update}
Let us first ignore the acyclicity constraint in \eqref{Problem:micp}, and consider solving problem \eqref{Problem:micp} with respect to a single variable $\Gamma_{uv}$, for $u,v=1,\ldots,m$, with the other coordinates of $\Gamma$ fixed. Specifically, we are solving 
\begin{equation}
\label{Problem:update}
    \min\limits_{\Gamma_{uv} \in \mathbb{R}} g(\Gamma_{uv}) \coloneqq \sum_{i=1}^m-2\log(\Gamma_{ii})+\mathrm{tr}\left(\Gamma\Gamma^\T\hat{\Sigma}\right) + \lambda^2\norm{\Gamma-\mathrm{diag}(\Gamma)}_{\ell_0},
\end{equation}
with $\Gamma_{ij}$ being fixed for $i\not=u, j\not=v$.

\begin{proposition}
\label{prop:update}
The solution to problem \eqref{Problem:update}, for $u,v=1,\ldots,m$ and $v\not = u$ is given by
\begin{align*}
\hat{\Gamma}_{uv} =
\begin{cases}
    \frac{-A_{uv}}{2\hat{\Sigma}_{uu}},& \text{if } \lambda^2 \leq \frac{A_{uv}^2}{4\hat{\Sigma}_{uu}},\\
    0,              & \text{otherwise}.
\end{cases}\quad ; \quad 
\hat{\Gamma}_{uu} = \frac{-A_{uu} + \sqrt{A_{uu}^2+16\hat{\Sigma}_{uu}}}{4\hat{\Sigma}_{uu}}, 
\end{align*}
where $A_{uu}=\sum\limits_{j\not=u}\Gamma_{ju}\hat{\Sigma}_{ju} + \sum\limits_{k\not=u}\Gamma_{ku}\hat{\Sigma}_{uk}$ and 
$A_{uv} = \sum\limits_{j\not=u}\Gamma_{jv}\hat{\Sigma}_{ju} + \sum\limits_{k\not=u}\Gamma_{kv}\hat{\Sigma}_{uk}.$
\end{proposition}
\begin{proof} 
The original objective function $g$ with $\ell_0$-norm is nonconvex and discontinuous. To find the optimal solution, we compare $g(0)$ with $g(\hat{\gamma}_{uv})$ for some $\hat{\gamma}_{uv} \not = 0$.

For any $u\in V$, we have 
\begin{align}
\label{eq:trace_decompose}
    \mathrm{tr}\left(\Gamma\Gamma^\T \hat{\Sigma}\right) &= \sum\limits_{i=1}^m\Gamma_{ui}\left(\Gamma_{ui}\hat{\Sigma}_{uu} +\sum_{j \not = u}\Gamma_{ji}\hat{\Sigma}_{ju}\right)+\sum\limits_{k\neq u}\sum\limits_{i=1}^m \Gamma_{ki}\left(\Gamma_{ui}\hat{\Sigma}_{uk} + \sum_{j\neq u} \Gamma_{ji}\hat{\Sigma}_{jk}\right).
\end{align}

We begin by analyzing off-diagonal entries $\Gamma_{uv}$ (where $u \not = v$). For each nonzero $\Gamma_{uv}$, optimizing $g(\Gamma_{uv})$ is equivalent to minimizing $\mathrm{tr}(\Gamma\Gamma^\T\hat{\Sigma})$. From equation~\eqref{eq:trace_decompose}, the derivative of $\mathrm{tr}(\Gamma\Gamma^\T \hat{\Sigma})$ with respect to $\Gamma_{uv}$ is given by
\begin{equation*}
\label{eqn:gradient_trace}
    \begin{aligned}
        \frac{\partial \mathrm{tr}(\Gamma\Gamma^\T \hat{\Sigma})}{\partial \Gamma_{uv}}=
    2\hat{\Sigma}_{uu} \Gamma_{uv} + \sum\limits_{j\not=u}\Gamma_{jv}\hat{\Sigma}_{ju} + \sum_{k\not=u}\Gamma_{kv}\hat{\Sigma}_{uk} = 2\hat{\Sigma}_{uu} \Gamma_{uv} + A_{uv}.
    \end{aligned}
\end{equation*}
Setting the derivative to zero, and defining
$\hat{\gamma}_{uv} := -A_{uv}/{(2\hat{\Sigma}_{uu})}$, we obtain 
\[
    \arg\min\limits_{\Gamma_{uv}} g(\Gamma_{uv})=\hat{\Gamma}_{uv}\coloneqq
\begin{cases}
    \hat{\gamma}_{uv},& \text{if } g(\hat{\gamma}_{uv}) \leq g(0),\\
    0,              & \text{otherwise}.
\end{cases}
\]
Furthermore, we observe that $\frac{\partial^2}{\partial \Gamma_{uv}^2} \mathrm{tr}(\Gamma\Gamma^\T\hat{\Sigma}) = 2\hat{\Sigma}_{uu} > 0$. This positive second derivative implies that optimizing $g(\Gamma_{uv})$ for any nonzero off-diagonal entry reduces to solving a convex optimization problem.

 Given that $g(\hat{\gamma}_{uv})$ represents the optimal objective value for any nonzero $\Gamma_{uv}$, comparing it with $g(0)$ allows us to determine the optimal solution.
Note that $g(\hat{\gamma}_{uv}) - g(0) = \hat{\gamma}_{uv}^2 \hat{\Sigma}_{uu} + \hat{\gamma}_{uv}A_{uv} + \lambda^2$. Thus, $g(\hat{\gamma}_{uv}) \leq g(0)$ is equivalent to $\lambda^2 \leq A_{uv}^2/{(4\hat{\Sigma}_{uu})}$.

Now we consider the update of $\Gamma_{uv}$ when $u = v$. We have:
\begin{equation*}
\label{eqn:zero_grad_uu}
    \frac{\partial g(\Gamma_{uu})}{\partial \Gamma_{uu}} = \frac{-2}{\Gamma_{uu}} + 2\hat{\Sigma}_{uu} \Gamma_{uu} + \sum\limits_{j\not=u}\Gamma_{ju}\hat{\Sigma}_{ju} + \sum\limits_{k\not=u}\Gamma_{ku}\hat{\Sigma}_{uk} = \frac{-2}{\Gamma_{uu}} + 2\hat{\Sigma}_{uu} \Gamma_{uu} + A_{uu}.
\end{equation*}
Setting ${\partial g(\Gamma_{uu})}/{\partial \Gamma_{uu}}  = 0$, we obtain:
$\hat{\Gamma}_{uu} = ({-A_{uu}\allowbreak+ (A_{uu}^2\allowbreak+16\hat{\Sigma}_{uu}})^{1/2})\allowbreak/{(4\hat{\Sigma}_{uu})}.
$
\end{proof}

\subsection{Accounting for acyclicity and full algorithm description}
\label{sec:accounting_acyclicity}

Algorithm \ref{algo:cd_spacer} fully describes our procedure. The input to our algorithm is the sample covariance $\hat{\Sigma}$, regularization parameter $\lambda \in \R_+$, a super-structure graph $E_{\mathrm{super}}$ that is a superset of edges that contains the true edges, a positive integer $C$, and a topological ordering $O$ that is a permutation of $\{1,2,\dots,m\}$. We allow the user to restrict the set of possible edges to be within a user-specified \emph{super-structure} set of edges $E_\text{super}$. A natural choice of the superstructure is the moral graph, which can be efficiently and accurately estimated via existing algorithms such as the graphical lasso \citep{Friedman07coordinate}, neighborhood selection \citep{Meinshausen2006HighdimensionalGA}, or $\ell_0$-regularized pseudo-likelihood-based estimator \citep{behdin2023sparse}. This superstructure could also be the complete graph if a reliable superstructure estimate is unavailable. The input topological ordering $O$ governs the order in which the coordinates of $\Gamma$ are updated in our coordinate descent algorithm (see Remark \ref{remark:different_ordering} for more discussion). 

We start by initializing $\Gamma$ according to input $\Gamma^\mathrm{init}$ and permute the rows and columns of $\hat{\Sigma}$ according to $O$. Then, for each pair of indices $u$ and $v$ ranging from 1 to $m$, we update $\Gamma_{uv}$ based on specific rules. If $u=v$ (a diagonal entry), we update it directly according to Proposition \ref{prop:update}. Among the off-diagonal entries, we only update those within the superstructure. Specifically, if $u \neq v$, and $(u,v)$ is in the superstructure, we check if setting $\Gamma_{uv}$ to a nonzero value violates the acyclicity constraint. (We use the breadth-first search algorithm  \cite[see][]{ellis2008learning,Fu13} to check for acyclicity.) If it does not, we update $\Gamma_{uv}$ as per Proposition \ref{prop:update}; otherwise, we set $\Gamma_{uv}$ to 0. 
We refer to a full sequence of coordinate updates as a full loop. 
The loop is repeated until convergence, when the objective values no longer improve after a complete loop. We keep track of the support of $\Gamma$s encountered during the algorithm. When the occurrence count of a particular support of $\Gamma$s reaches a predefined threshold, $C$, a spacer step \citep{bertsekas2016nonlinear,Rahul20coordinate} is initiated, during which we update every nonzero coordinate iteratively.
Note that in the spacer step, we use $\hat{\gamma}_{uv}$, which is the optimal update without considering the sparsity penalty, that is, we use $\lambda^2 = 0$. 
The use of spacer steps stabilizes the behavior of updates and ensures convergence.  After finishing the spacer step, we reset the counter of the support of the current solution.

\begin{remark}
The effect of the input ordering $O$ is that Algorithm~\ref{algo:cd_spacer} iteratively updates the coordinates of $\Gamma$ in a particular order: updating row by row in the order given by ${O}$, and within each row, according to ${O}$. In Section \ref{sec:theory}, we show that regardless of the input ordering $O$, as well as the initialization $\Gamma_{\mathrm{init}}$, Algorithm~\ref{algo:cd_spacer} converges, and is asymptotically optimal (matches the optimal objective value as the sample size tends to infinity). However, while different orderings result in similar final objective values, the associated Markov equivalence classes can vary significantly. Thus, a given ordering or initialization may result in a Markov equivalence class that is far from the population one. In Section \ref{sec:selecting_ordering}, we discuss this challenge, and advocate for using an existing algorithm to find a suitable ordering, and use this ordering as input to Algorithm~\ref{algo:cd_spacer}. Finally, we note that the input topological ordering $O$ is simply an initialization to Algorithm~\ref{algo:cd_spacer}. Indeed, the topological ordering corresponding to the output $\hat{\Gamma}$ can in general be different than $O$. Finally, we note that the default initialization $\Gamma^{\mathrm{init}} = I$ tends to perform well in practice and is comparable to other initialization values, especially when a good ordering $O$ is selected; see Section~\ref{sec:experiments}.
\label{remark:different_ordering}
\end{remark}

\begin{algorithm}[tb]
\caption{Cyclic coordinate descent algorithm with spacer steps}
\label{algo:cd_spacer}
\begin{algorithmic}[1]
    \STATE \textbf{Input:} Sample covariance $\hat{\Sigma}$, regularization parameter $\lambda \in \R_+$, super-structure $E_{\mathrm{super}}$, positive integer $C$, any topological ordering $O$ that is a permutation of $\{1,2,\dots,m\}$, initial matrix $\Gamma^{\mathrm{init}}$ (default: identity matrix).
    \STATE \textbf{Initialize:} $\Gamma^{0} \gets \Gamma^{\mathrm{init}}$ ; $t\gets 1$;
    \STATE $\hat{\Sigma}\gets \hat{\Sigma}$ with rows and columns permuted according to $O$.
    \WHILE{objective function $f(\Gamma^{t})$ continue decreasing}
        \FOR{$u=1$ to $m$} \label{alg:line4}
            \STATE $\Gamma^{t}_{uu} = \hat{\Gamma}_{uu}$, where $\hat{\Gamma}_{uu}$ is calculated from Proposition \ref{prop:update} using the recently updated $\Gamma^t$. 
            \FOR{$v=1$ to $m$ such that $(u, v) \in E_{\mathrm{super}}$}
                \STATE If $\Gamma^{t}_{uv}\not = 0$  violates acyclicity constraints, set $\Gamma^{t}_{uv} = 0$.
                \STATE If $\Gamma^{t}_{uv}\not = 0$  would not violate acyclicity constraints, set $\Gamma^{t}_{uv}=\hat{\Gamma}_{uv}$.
                \STATE $t\gets t+1$.
                \STATE $\mathrm{Count}[\mathrm{support}(\Gamma^t)] \gets \mathrm{Count}[\mathrm{support}(\Gamma^t)] + 1$.
                \IF{$\mathrm{Count}[\mathrm{support}(\Gamma^t)] = Cm^2$}
                    \STATE    ~~ $\Gamma^{t+1} \gets \mathrm{SpacerStep}(\Gamma^t)$. \hspace{1cm} (Algorithm~\ref{algo:spacer})\\
                    ~~ $\mathrm{Count}[\mathrm{support}(\Gamma^t)] = 0.$\\
                    \STATE $t\gets t+1$.
                \ENDIF
            \ENDFOR
        \ENDFOR
        \label{alg:line8}
    \ENDWHILE
    \STATE $\hat{\Gamma} \gets \Gamma^{t}$ with rows and columns reordered back to the original variable order.
    \STATE \textbf{Output:} $\hat{\Gamma}$ and the Markov equivalence class $\mathrm{MEC}(\mathcal{G}(\hat{\Gamma}-\mathrm{diag}(\hat{\Gamma})))$.
\end{algorithmic}
\end{algorithm}

\begin{algorithm}[tb]
\caption{SpacerStep}
    \begin{algorithmic}[1]
        \STATE \textbf{Input: } $\Gamma^t$
        \FOR{$(u,v)  \in \mathrm{support}(\Gamma^t)$}
            \STATE Set  $\Gamma^{t+1}_{uv} \leftarrow \hat{\gamma}_{uv}$
        \ENDFOR
        \STATE \textbf{Output: } $\Gamma^{t+1}$
    \end{algorithmic}
    \label{algo:spacer}
\end{algorithm}

\section{Convergence and Optimality Guarantees}
\label{sec:theory}
We provide convergence and optimality guarantees for our coordinate descent procedure (Algorithm~\ref{algo:cd_spacer}). Specifically, we follow a similar proof strategy as \cite{Rahul20coordinate} to show that Algorithm~\ref{algo:cd_spacer} converges. Remarkably, we also prove the surprising result that the objective value attained by our coordinate descent algorithm provably converges to the optimal objective value of \eqref{Problem:micp}. Throughout, we assume the super-structure $E_\mathrm{super}$ that is supplied as input to Algorithm~\ref{algo:cd_spacer} satisfies $E^\star \subseteq E_\mathrm{super}$ where $E^\star$ denotes the true edge set; see \cite{xu2024integer} for a discussion on how the graphical lasso can yield super-structures that satisfy this property with high probability.

Our analysis relies on the notion of coordinate-wise minimum \citep{Rahul20coordinate}, which represent solutions Algorithm~\ref{algo:cd_spacer} obtains under different input orderings $O$.
\begin{definition}(Coordinate-wise (CW) minimum)
A connectivity matrix $\Gamma^\mathrm{CW} \in \R^{m\times m}$ of a DAG is the CW minimum of problem \eqref{Problem:micp} if for every $(u,v), u,v=1,\ldots,m$, $\Gamma^\mathrm{CW}_{uv}$
is a minimizer of $g(\Gamma_{uv})$ with other coordinates of $\Gamma^\mathrm{CW}$ held fixed.
\label{defn:coordinate_wise_minimum}
\end{definition}

\subsection{Convergence}
\label{sec:covergence_optimalith}
Our convergence analysis requires an assumption on the sample covariance matrix:
\begin{assumption}(Positive definite sample covariance)
\label{assumption:convexity}
    The sample covariance matrix $\hat{\Sigma}$ is positive definite.
\end{assumption}
Assumption~\ref{assumption:convexity} is satisfied almost surely if $n \geq m$ and the samples of the random vector $X$ are generated from an absolutely continuous distribution. Under this mild assumption, our coordinate descent algorithm provably converges to a coordinate-wise minimum.
\begin{theorem}(Convergence of Algorithm~\ref{algo:cd_spacer})
\label{thm:convergence}
Let $\{\Gamma^{t}\}_{t=1}^\infty$ be the sequence of estimates generated by Algorithm \ref{algo:cd_spacer}. Suppose that Assumption \ref{assumption:convexity} holds. Then, for input topological ordering $O$ or initial point $\Gamma^{\mathrm{init}}$ to Algorithm \ref{algo:cd_spacer}:
\begin{enumerate}
    \item the sequence $\{\mathrm{support}(\Gamma^{t})\}_{t=1}^\infty$ stabilizes after a finite number of iterations; that is, there exists a positive integer $M$ and a
support set $\hat{E} \subseteq \{(i,j): i,j = 1,2,\dots,m\}$ such that $\mathrm{support}(\Gamma^t) = \hat{E}$ for all $t \geq M$.
    \item the sequence $\{\Gamma^{t}\}_{t=1}^\infty$ converges to a coordinate-wise minimum $\Gamma$ with $\mathrm{support}(\Gamma) = \hat{E}$.
\end{enumerate}
\end{theorem}
{The proof of Theorem~\ref{thm:convergence} is provided in Appendix~\ref{proof:convergence} and adapts the analysis in \cite{Rahul20coordinate}, which was originally developed for coordinate descent in variable selection, to our problem setting.}

\subsection{Optimality guarantees}
\label{sec:optimality}

Let $d_\text{max} := \max_i |\{j: (j,i) \in E_\mathrm{super}\}|$, the maximum degree of a node in $E_\mathrm{super}$. Our analysis for optimality guarantees requires an assumption on the population model. For the set $E \subseteq \{(i,j): i,j = 1,2\dots,m\}$, consider the optimization problem
\begin{equation}
\Gamma^\star_E = \argmin_{\Gamma \in \mathbb{R}^{m \times m}}  \sum_{i=1}^m-2\log(\Gamma_{ii})+\mathrm{tr}\left(\Gamma\Gamma^\T{\Sigma}^\star\right) \quad
\text { s.t. } \quad \mathrm{support}(\Gamma) \subseteq {E}.
\label{Problem:support_pop}
\end{equation}
\begin{assumption} There exists constants $\bar{\kappa}, \underline{\kappa} >0$ such that $\sigma_\text{min}(\Gamma^\star_E) \geq \underline{\kappa}$ and $\sigma_{\text{max}}(\Gamma^\star_E) \leq \bar{\kappa}$ for every $E \subseteq E_{\mathrm{super}}$ where the graph $(V,E)$ is a DAG, where $\sigma_\text{min}(\cdot)$ and $\sigma_\text{min}(\cdot)$ are the smallest and largest eigenvalues respectively.
\label{condition:rest_supp}
\end{assumption}

\begin{theorem}
\label{thm:obj}
    Let $\hat{\Gamma}$ be the solution of Algorithm \ref{algo:cd_spacer} with any input topological ordering $O$ or initial point $\Gamma^{\mathrm{init}}$, and let $\hat{\Gamma}^{\mathrm{opt}}$ be an optimal solution of \eqref{Problem:micp}. Suppose Assumption~\ref{condition:rest_supp}
 holds and let the regularization parameter be chosen so that $\lambda^2 = \mathcal{O}(\log{m}/n)$ where $m$ and $n$ denote the number of nodes and number of samples, respectively. Then, 
    \begin{enumerate}
    \item $f(\hat{\Gamma})-f(\hat{\Gamma}^{\mathrm{opt}}) \to_P 0$ as $n\rightarrow \infty$,
    \item if $n/\log(n) \geq \mathcal{O}(m^2\log{m})$, with probability greater than $1-1/\mathcal{O}(n)$, we have 
$0 \leq f(\hat{\Gamma}) - f(\hat{\Gamma}^{\mathrm{opt}}) \leq \mathcal{O}(\sqrt{d_{\text{max}}^2m^4\log{m}/n})$.
    \end{enumerate}
   In other words, regardless of the input ordering $O$, the objective value of the coordinate descent solution converges in probability to the optimal objective value as $n \to \infty$. Further, assuming the sample size $n$ is sufficiently large, with high probability, the difference in objective value is bounded by $\mathcal{O}(\sqrt{d_{\text{max}}^2m^4\log{m}/n})$. 
\end{theorem}

We prove Theorem~\ref{thm:obj} by showing the results for any coordinate-wise minimum of problem \eqref{Problem:micp} (see Definition~\ref{defn:coordinate_wise_minimum}). Our proof relies on the following lemmas. Throughout, we let $\hat{E}$ be the support of $\hat{\Gamma}$, that is, $\hat{E} = \{(i,j), \hat{\Gamma}_{ij}\neq 0\}$.
\begin{lemma}
    \label{lem:equal_trace}
    Let $\hat{\Gamma}, \hat{\Gamma}^{\mathrm{opt}}$ be the solution of Algorithm \ref{algo:cd_spacer} and optimal solution of \eqref{Problem:micp}, respectively. Then, i) for any $u,v = 1,2,\dots,m,  A_{uv} + 2\hat{\Gamma}_{uv}\hat{\Sigma}_{uu} = 2 (\hat{\Sigma}\Gamma)_{uv}$ where $A_{uv}$ is defined in Proposition~\ref{prop:update}.
     ii) if $\hat{\Gamma}_{uv} \neq 0$, then $(\hat{\Sigma}\hat{\Gamma})_{uv}=0$, and iii) the matrix $\hat{\Gamma}\hat{\Gamma}^\T\hat{\Sigma}$ has ones on the diagonal. Moreover, these properties also hold for the optimal solution $\hat{\Gamma}^{\mathrm{opt}}$.
\end{lemma}
\begin{proof}[Proof of Lemma \ref{lem:equal_trace}]
For $u,v = 1, \ldots, m$, by the definition of  $A_{uv}$,   $A_{uv} + 2\Gamma_{uv}\hat{\Sigma}_{uu} = 2 (\hat{\Sigma}\Gamma)_{uv}$, proving item i. Since any solution from Algorithm~\ref{algo:cd_spacer}, $\hat{\Gamma}$ satisfies Proposition~\ref{prop:update}, for any $(u,v)\in \hat{E}$, $(4\hat{\Sigma}_{uu}\hat{\Gamma}_{uu} + A_{uu})^2 = A_{uu}^2 + 16\hat{\Sigma}_{uu}$ and $A_{uv} = -2\hat{\Gamma}_{uv}\hat{\Sigma}_{uu}$. Combining the previous relations, we conclude that $(\hat{\Sigma}\hat{\Gamma})_{uv} = 0$. Therefore, for any $(u,v)\in \hat{E}$, we have $\hat{\Gamma}_{uv} \not = 0$ and $(\hat{\Sigma}\hat{\Gamma})_{uv} = 0$, resulting in $\hat{\Gamma}_{uv}(\hat{\Sigma}\hat{\Gamma})_{uv} = 0$. This proves item ii. Plugging $A_{uu}$ into the previous relations, we arrive at $\hat{\Gamma}_{uu}(\hat{\Sigma}\hat{\Gamma})_{uu} = 1$. Thus,       $(\hat{\Gamma}\hat{\Gamma}^\T \hat{\Sigma})_{ii} = \sum_{j=1}^m \hat{\Gamma}_{ij} (\hat{\Gamma}^\T \hat{\Sigma})_{ji} = \hat{\Gamma}_{ii} (\hat{\Gamma}^\T \hat{\Sigma})_{ii} = 1$, proving item iii. Since $\hat{\Gamma}^{\mathrm{opt}}$ is an optimal solution, it must satisfy the conditions stated in Proposition~\ref{prop:update}; otherwise, we could update $\hat{\Gamma}^{\mathrm{opt}}$ to achieve a better objective value. Therefore, the same arguments apply to $\hat{\Gamma}^{\mathrm{opt}}$.
\end{proof}

\begin{lemma} Let $E \subseteq \{(i,j): i,j = 1,2,\dots,m\}$ be any set where the graph indexed by tuple $(V,E)$ is a DAG. Consider the estimator:
\begin{eqnarray}
\hat{\Gamma}_E = \argmin_{\Gamma \in \mathbb{R}^{m \times m}} \sum_{i=1}^m-2\log(\Gamma_{ii})+\mathrm{tr}\left(\Gamma\Gamma^\T\hat{\Sigma}\right)\quad \text{s.t.}\quad \mathrm{support}(\Gamma) \subseteq E. 
\label{Problem:support}
\end{eqnarray}
Suppose that $4\sqrt{m}\bar{\kappa}\|\hat{\Sigma}-\Sigma^\star\|_2   \leq \min\{\underline{\kappa}/(8\sqrt{m}\bar{\kappa}^2),\frac{1}{2\bar{\kappa}}\}$ and that $\hat{\Sigma}$ is positive definite. Then, $\|\hat{\Gamma}_E - \Gamma^\star_E\|_F\leq 4\sqrt{m}\bar{\kappa}\|\hat{\Sigma}-\Sigma^\star\|_2$.
\label{lemma:concentration_Gamma}
\end{lemma}
\begin{proof}[Proof of Lemma \ref{lemma:concentration_Gamma}] The proof follows from standard convex analysis and Brouwer's fixed point theorem; we provide the details below. Since $\Gamma$ follows a DAG structure, the objective of \eqref{Problem:support} can be written as:
$-2\log\det(\Gamma) + \|\Gamma\hat{\Sigma}^{1/2}\|_F^2$. 
The KKT conditions state that there exists $Q$ with $\text{support}(Q) \cap E = \emptyset$ such that the optimal solution $\hat{\Gamma}_E$ of \eqref{Problem:support} satisfies $-2\hat{\Gamma}_E^{-1}+Q + 2\hat{\Gamma}_E\hat{\Sigma} = 0$ and $\mathrm{support}(\hat{\Gamma}_E) \subseteq {E}$. Let $\Delta = \hat{\Gamma}_E - \Gamma^\star_E$. By Taylor series expansion, 
$\hat{\Gamma}_E^{-1} = (\Gamma^\star_E+\Delta)^{-1} = {\Gamma^\star_E}^{-1} - {\Gamma^\star_E}^{-T}\Delta {\Gamma^\star_E}^{-1}+\mathcal{R}(\Delta),$
where $\mathcal{R}(\Delta)= 2{\Gamma^\star_E}^{-1}\sum_{k=2}^{\infty}(-\Delta{\Gamma^\star_E})^k$. For any matrix $M\in\mathbb{R}^{m\times{m}}$, define the operator $\mathbb{I}^\star$ with $\mathbb{I}^\star(M) := 2{\Gamma^\star_E}^{-T}M{\Gamma^\star_E}^{-1} + 2M\Sigma^\star$. Let $\mathcal{K}$ be the subspace $\mathcal{K} = \{M \in \mathbb{R}^{m \times m}: \text{support}(M) \subseteq E\}$ and let $P_{\mathcal{K}}$ be the projection operator onto subspace $\mathcal{K}$ that zeros out entries of the input matrix outside of the support set $E$. From the optimality condition of \eqref{Problem:support_pop}, we have $\mathcal{P}_{\mathcal{K}}[2{\Gamma_E^\star}^{-1}-2\Gamma^\star_E\Sigma^\star] = 0$. Then, the optimality condition of \eqref{Problem:support} can be rewritten as:
\begin{eqnarray}
\begin{aligned}
&\mathcal{P}_{\mathcal{K}}\left[\mathbb{I}^\star(\Delta)+2\Delta(\hat{\Sigma}-\Sigma^\star)-2\mathcal{R}(\Delta)+H_n\right] = 0,
\end{aligned}
\label{eqn:optim_cond}
\end{eqnarray}
where $H_n = 2\Gamma_E^*(\hat{\Sigma}-\Sigma^*)$.
Since $\hat{\Gamma}_E \in \mathcal{K}$ and ${\Gamma}^\star_E \in \mathcal{K}$, we have that $\Delta \in \mathcal{K}$. We use Brouwer's theorem to obtain a bound on $\|\Delta\|_F$. We define an operator $J$ as $\mathcal{K} \to \mathcal{K}$:
$$J(\delta) = \delta-(\mathcal{P}_{\mathcal{K}}\mathbb{I}^\star\mathcal{P}_{\mathcal{K}})^{-1}\left(\mathcal{P}_{\mathcal{K}}\left[\mathbb{I}^\star\mathcal{P}_{\mathcal{K}}(\delta)-2\mathcal{R}(\delta)+H_n+2\delta(\hat{\Sigma}-\Sigma^\star)\right]\right).$$
Here, the operator $\mathcal{P}_{\mathcal{K}}\mathbb{I}^\star\mathcal{P}_{\mathcal{K}}$ is invertible since $\sigma_\text{min}(\mathbb{I}^\star) = \sigma_\text{min}({\Gamma^\star_E}^{-1})^2 \geq \frac{1}{\bar{\kappa}^2}$. Notice that any fixed point $\delta$ of $J$ satisfies the optimality condition \eqref{eqn:optim_cond}. Furthermore, since the objective of \eqref{Problem:support} is strictly convex, we have that the fixed point must be unique. In other words, the unique fixed point of $J$ is given by $\Delta$. Now consider the following compact set: $\mathcal{B}_r = \{\delta \in \mathbb{R}^{m \times m}: \text{support}(\delta) \subseteq E, \|\delta\|_F \leq r\}$ for $r = 4\sqrt{m}\bar{\kappa}\|\hat{\Sigma}-\Sigma^\star\|_2  $. By the assumption,  $r \leq \min\{\underline{\kappa}/(8\sqrt{m}\bar{\kappa}^2),\frac{1}{2\bar{\kappa}}\}$. Then, for every $\delta \in \mathcal{B}_r$, we have 
 $\|\delta\Gamma^\star_E\|_F \leq \bar{\kappa}r \leq 1/2$,  and $\|\mathcal{R}(\delta)\|_F \leq 2\sqrt{m}\|\Gamma^\star_E\|_2^2/\sigma_{\text{min}}(\Gamma^\star_E)\|\delta\|_2^2\frac{1}{1-\|\delta\Gamma^\star_E\|_2} \leq 2\sqrt{m}\bar{\kappa}^2r^2/\underline{\kappa}\frac{1}{1-r\bar{\kappa}} \leq 4\sqrt{m}\bar{\kappa}^2r^2/\underline{\kappa}$. Since $\|H_n\|_F \leq 2\sqrt{m}\|\Gamma^\star_E\|_2\|\hat{\Sigma}-\Sigma^\star\|_2$ and $\|J(\delta)\|_F \leq \bar{\kappa}^2[\|H_n\|_F + 2\|\mathcal{R}(\delta)\|_F+2\|\delta(\hat{\Sigma}-\Sigma^\star)\|_F]$, we conclude that $\|J(\delta)\|_F \leq \frac{8\sqrt{m}\bar{\kappa}^2r^2}{\underline{\kappa}} + \frac{\bar{\kappa}^2r}{2} + \frac{r^2}{2\sqrt{m}\bar{\kappa}} \leq r$. In other words, we have shown that $J$ maps $\mathcal{B}_r$ onto itself. Appealing to Brouwer's fixed point theorem, we conclude that the fixed point must also lie inside $\mathcal{B}_r$. Thus, we conclude that $\|\Delta\|_F \leq r$.

\end{proof}
\begin{lemma}With probability greater than $1-1/\mathcal{O}(n)$, we have $\|\hat{\Sigma}-\Sigma^\star\|_2 \leq \mathcal{O}(\sqrt{m\log(n)/n})$,   $\|\hat{\Sigma}\|_\infty \leq 2\bar{\kappa}^2$, $\sigma_\text{min}(\hat{\Sigma}) \geq \underline{\kappa}^2/2$, $\|\hat{\Gamma}\|_\infty \leq 2\bar{\kappa}$ and $\sigma_\text{min}(\hat{\Gamma}) \geq \underline{\kappa}/2$.
\label{lemma:finite_sample}
\end{lemma}
\begin{proof}[Proof of Lemma~\ref{lemma:finite_sample}]
From standard Gaussian concentration results, when $n/\log(n) \geq \mathcal{O}(m)$, with probability greater than $1-\mathcal{O}(1/n)$, we have that $\|\hat{\Sigma}-\Sigma^\star\|_2 \leq \mathcal{O}\allowbreak(\sqrt{m\log(n)/n})$. By  Assumption~\ref{condition:rest_supp}, with probability greater than $1-\mathcal{O}(1/n)$, $\hat{\Sigma}$ is positive definite,  with $\|\hat{\Sigma}\|_\infty \leq 2\bar{\kappa}^2$ and $\sigma_\text{min}(\hat{\Sigma}) \geq \underline{\kappa}^2-\mathcal{O}(\sqrt{m\log(n)/n}) \geq \underline{\kappa}^2/2$. Furthermore, appealing to Lemma~\ref{lemma:concentration_Gamma},  
$\|\hat{\Gamma}-{\Gamma}^\star_{\hat{E}}\|_F \leq \mathcal{O}(\sqrt{m^2\log(n)/n})$. Thus, $\|\hat{\Gamma}\|_\infty \leq \|\Gamma^\star_{\hat{E}}\|_2 + \bar{\kappa} \leq 2\bar{\kappa}$ and $\sigma_\text{min}(\hat{\Gamma}) \geq \underline{\kappa}-\mathcal{O}(\sqrt{m\log(n)/n})\geq \underline{\kappa}/2$. 
\end{proof}

\begin{proof}[Proof of Theorem~\ref{thm:obj}] \textbf{Part 1)}. First, observe that
 by definition, $\hat{\Gamma}^\mathrm{opt}$ is the optimal solution of the problem \begin{equation*}
    \min\limits_{\Gamma \in \mathbb{R}^{m \times m}}  f(\Gamma)\quad 
\text{s.t.} \quad \mathcal{G}\left(\Gamma-\mathrm{diag}\left(\Gamma\right)\right) \text{ is a DAG},
\end{equation*}
    where $f(\Gamma):= \sum_{i=1}^m-2\log(\Gamma_{ii})+\mathrm{tr}(\Gamma\Gamma^\T \hat{\Sigma}) + \lambda^2\norm{\Gamma-\mathrm{diag}(\Gamma)}_{\ell_0}$ is the penalized negative log-likelihood function. The negative log-likelihood part of $f$ is convex with the optimal value being $\min_{\Theta}\{-\log\det(\Theta)+\mathrm{tr}(\Theta\hat{\Sigma})\}$.  Clearly, $\sum_{i=1}^m-2\log(\hat{\Gamma}^{\mathrm{opt}}_{ii})+\mathrm{tr}(\hat{\Gamma}^{\mathrm{opt}}(\hat{\Gamma}^{\mathrm{opt}})^\T \hat{\Sigma}) \geq \min_{\Theta}\{-\log\det(\Theta)+\mathrm{tr}(\Theta\hat{\Sigma})\}$. Therefore, 
    \begin{equation}\label{eq:fgammaopt}
    f(\hat{\Gamma}^{\mathrm{opt}}) \geq \min_{\Theta}\{-\log\det(\Theta)+\mathrm{tr}(\Theta\hat{\Sigma})\}= \log\det(\hat{\Sigma})+m.
    \end{equation}
Now, 
\begin{align*}
        0\leq f(\hat{\Gamma}) - f(\hat{\Gamma}^{\mathrm{opt}}) \leq f(\hat{\Gamma}) - \log\det(\hat{\Sigma}) - m =  -\log\det(\hat{\Gamma}\hat{\Gamma}^\T\hat{\Sigma}) + \lambda^2\|\hat{\Gamma}-\text{diag}(\hat{\Gamma})\|_0 , 
    \end{align*}
where the second inequality follows from \eqref{eq:fgammaopt}, and 
the equality follows from appealing to item iii.  of Lemma \ref{lem:equal_trace} to conclude that $f(\hat{\Gamma}) = -\log\det(\hat{\Gamma}\hat{\Gamma}^\T) + m + \lambda^2\|\hat{\Gamma}-\text{diag}(\hat{\Gamma})\|_0$, because  $\mathrm{tr}(\Theta\hat{\Sigma})\}=m$. 

Our strategy is to show that as $n \to \infty$, $\hat{\Gamma}\hat{\Gamma}^\T\hat{\Sigma}$ converges to a matrix with ones on the diagonal and whose off-diagonal entries induce a DAG. Thus, $\log\det(\hat{\Gamma}\hat{\Gamma}^\T\hat{\Sigma}) \rightarrow \log \prod_{i=1}^m 1 = 0$ as $n \to \infty$. Since $\lambda^2 \to 0$ as $n \to \infty$ and $\|\hat{\Gamma}-\text{diag}(\hat{\Gamma})\|_0\leq m^2$, we can then conclude the desired result. For any $u,v = 1,2,\ldots,m$:
\begin{align}
    (\hat{\Gamma}\hat{\Gamma}^\T\hat{\Sigma})_{uv} = \sum_{i=1}^m  \hat{\Gamma}_{ui}(\hat{\Sigma}\hat{\Gamma})_{vi}= \hat{\Gamma}_{uu}(\hat{\Sigma}\hat{\Gamma})_{vu} + \hat{\Gamma}_{uv}(\hat{\Sigma}\hat{\Gamma})_{vv} + \sum_{i\in F_{uv}} \hat{\Gamma}_{ui}(\hat{\Sigma}\hat{\Gamma})_{vi},
    \label{eqn:F_equation}
\end{align}
where $F_{uv}:=\{i \mid i\not=u, i\not=v, (u, i)\in \hat{E}, (v, i) \not \in \hat{E}\}$. Here, the second equality is due to item ii.\ of Lemma~\ref{lem:equal_trace}; note that if $\hat{\Gamma}_{ui}(\hat{\Sigma}\hat{\Gamma})_{vi} \neq 0$, then $i\in F_{uv}$ as otherwise either $\hat{\Gamma}_{ui} = 0$ or $(\hat{\Sigma}\hat{\Gamma})_{vi} = 0$. We consider the two possible settings for $(u,v), u\neq v$: Setting I) $(u,v) \in \hat{E}$ which implies that $(v,u) \not\in \hat{E}$ as $\hat{\Gamma}$ specifies a DAG, and Setting II) $(u,v), (v,u) \not\in \hat{E}$. (Note that $(u,v),(v,u)\in \hat{E}$ is not possible since $\hat{\Gamma}$ specifies a DAG.)

\vspace{0.1in}
\noindent\underline{Setting I}: Since $(u,v) \in \hat{E}$ and $(v,u)\not\in \hat{E}$, we have
\begin{align*}
    (\hat{\Gamma}\hat{\Gamma}^\T\hat{\Sigma})_{vu} = \sum_{i\in F_{vu}} \hat{\Gamma}_{vi}(\hat{\Sigma}\hat{\Gamma})_{ui}
    = \sum_{i\in F_{vu}} \hat{\Gamma}_{vi}\left(\frac{1}{2}A_{ui} +\hat{\Gamma}_{ui}\hat{\Sigma}_{uu}\right)
   {=}  \sum_{i\in F_{vu}} \frac{1}{2}\hat{\Gamma}_{vi}A_{ui}.
\end{align*}

Here, the first equality follows from appealing to \eqref{eqn:F_equation}, and noting that $\hat{\Gamma}_{vu} = 0$ and that $(\hat{\Sigma}\hat{\Gamma})_{uv} = 0$ according to item ii.\ of Lemma~\ref{lem:equal_trace}; the second equality follows from item i. of Lemma~\ref{lem:equal_trace}; the final equality follows from noting that $\hat{\Gamma}_{ui} = 0$ for $i \in F_{vu}$. 

For each $i\in F_{vu}$, Figure~\ref{fig:dags_diff_sett} (left) represents the relationships between the nodes $u,v,i$. Here, the directed edge from $u$ to $v$ from the constraint $(u,v) \in \hat{E}$ is represented by a dashed line, the directed edge from $v$ to $i$ from the constraint $i \in F_{vu}$ is represented by a solid line, and the directed edge that is disallowed due to the constraint $i \in F_{vu}$ is represented via a crossed-out solid line.

Since there is a directed path from $u$ to $i$, to avoid a cycle, a directed path from $i$ to $u$ cannot exist. Thus, adding the edge from $u$ to $i$ to $\hat{E}$ does not violate acyclicity and the fact that it is missing is due to $\lambda^2 > A_{ui}^2/(4\hat{\Sigma}_{uu})$ according to Proposition~\ref{prop:update}. Then, appealing to Lemma~\ref{lemma:finite_sample}, we conclude that with probability greater than $1-\mathcal{O}(1/n)$: $|(\hat{\Gamma}\hat{\Gamma}^\T\hat{\Sigma})_{vu}| \leq \sum_{i\in F_{vu}}\frac{1}{2}|\hat{\Gamma}_{vi}|2\lambda(\hat{\Sigma}_{uu})^{1/2} \leq 4\lambda\bar{\kappa}^2d_\text{max}$. 
In other words, in this setting, $|(\hat{\Gamma}\hat{\Gamma}^\T\hat{\Sigma})_{vu}| \to 0$ as $n \to \infty$.

\noindent\underline{Setting II}: Since $(u,v), (v,u) \not \in \hat{E}$, we have
\begin{align}
    (\hat{\Gamma}\hat{\Gamma}^\T\hat{\Sigma})_{uv} &= \hat{\Gamma}_{uu}\left(\frac{1}{2}A_{vu} + \hat{\Gamma}_{vu}\hat{\Sigma}_{vv}\right) + \sum_{i\in F_{uv}} \hat{\Gamma}_{ui}\left(\frac{1}{2}A_{vi} + \hat{\Gamma}_{vi}\hat{\Sigma}_{vv}\right) = \sum_{\substack{i\in F_{uv}\\\cup \{u\}}} \frac{\hat{\Gamma}_{ui}A_{vi}}{2}.
    \label{eqn:setting_2_eq}
\end{align}
Here, the first equality follows from plugging zero for $\hat{\Gamma}_{uv}$ in \eqref{eqn:F_equation} and appealing to item i. of Lemma~\ref{lem:equal_trace}; the second equality follows from plugging in zero for $\hat{\Gamma}_{vi}$ and $\hat{\Gamma}_{vu}$. Since $\hat{\Gamma}$ specifies a DAG, 
there cannot simultaneously be a directed path from $u$ to $v$ and from $v$ to $u$. Thus, either directed edges $(u,v)$ or $(v,u)$ can be added without creating a cycle. We consider the three remaining sub-cases below:

\vspace{0.1in}
\noindent\underline{Setting II.1. Adding $(u, v)$ to $\hat{E}$ violates acyclicity but adding $(v, u)$ does not}.

For each $i \in F_{uv}$, Figure\allowbreak~\ref{fig:dags_diff_sett} (middle) represents the relations between nodes $u, v$, and $i$. Here, due to the condition of Setting II, nodes $u$ and $v$ are not connected by an edge, which is displayed by a solid, crossed-out, undirected edge. Furthermore, the directed edge from $u$ to $i$ from the constraint $i \in F_{uv}$ is represented via a solid directed edge, the directed edge $v$ to $i$ that is disallowed due to the constraint $i \in F_{uv}$ is represented via a crossed-out solid line. Finally, the directed edge $u$ to $v$ that is disallowed due to acyclicity is represented via a crossed-out dashed line. 

Since adding the directed edge $(u,v)$ to $\hat{E}$ creates a cycle, then we have the following implications: i. adding $(v,u)$ to $\hat{E}$ does not violate acyclicity (as both edges $u \to v$ and $v \to u$ cannot simultaneously create cycles) and ii. there must be a directed path from $v$ to $u$. Implication i. allows us to conclude that $\hat{\Gamma}_{vu}$ must be equal to zero due to the condition $4\hat{\Sigma}_{vv}\lambda^2 > A_{vu}^2$ from Proposition~\ref{prop:update}. Combining implication ii. and the fact that there is a directed edge from $u$ to $i$ in $\hat{E}$ allows us to conclude that there cannot be a directed path from $i$ to $v$ as we would be creating a direct path from $u$ to itself. Thus, the fact that the directed edge $(v,i)$ is not in $\hat{E}$, or equivalently that $\hat{\Gamma}_{vi} = 0$, is due to $4\hat{\Sigma}_{vv}\lambda^2 > A_{vi}^2$ according to Proposition~\ref{prop:update}.  From \eqref{eqn:setting_2_eq} and Lemma~\ref{lemma:finite_sample}, we conclude with probability greater than $1-\mathcal{O}(1/n)$, $|(\hat{\Gamma}\hat{\Gamma}^\T\hat{\Sigma})_{uv}|\leq\sum_{i\in F_{vu}\cup \{u\}}\frac{1}{2}|\hat{\Gamma}_{vi}|2\lambda(\hat{\Sigma}_{uu})^{1/2} \leq 4\bar{\kappa}\lambda(1+d_\text{max})$.  In other words, in this setting, $|(\hat{\Gamma}\hat{\Gamma}^\T\hat{\Sigma})_{uv}| \to 0$ as $n \to \infty$.

\vspace{0.1in}
\noindent\underline{Setting II.2. Adding $(u, v)$ or $(v,u)$ to $\hat{E}$ would not violate acyclicity.} 

For each $i\in F_{uv}$, Figure~\ref{fig:dags_diff_sett} (right) represents the relations between the nodes $u, v$, and $i$.  Here, due to the condition of Setting II, nodes $u$ and $v$ are not connected by an edge, so this is displayed by a solid crossed-out undirected edge. Furthermore, the directed edge from $u$ to $i$ from the constraint $i \in F_{uv}$ is represented via a solid directed edge, the directed edge $v$ to $i$ that is disallowed due to the constraint $i \in F_{uv}$ is represented via a crossed-out solid line.

In this setting, recall that the directed edges $u$ to $v$ and $v$ to $u$ are not present in the estimate $\hat{E}$. Since neither of these two edges violates acyclicity according to the condition of this setting, we conclude that $4\hat{\Sigma}_{vv}\lambda^2 > A_{vu}^2$. There cannot be a path from $i$ to $v$ because then there would exist a path from $u$ to $v$, which contradicts the scenario that an edge from $v$ to $u$ does not create a cycle. As a result, an edge from $v$ to $i$ does not create a cycle and $\hat{\Gamma}_{vi} = 0$ is due to $4\hat{\Sigma}_{vv}\lambda^2 > A_{vi}^2$ according to Proposition~\ref{prop:update}. Thus, from \eqref{eqn:setting_2_eq} and Lemma~\ref{lemma:finite_sample}, we conclude that, with probability greater than $1-\mathcal{O}(1/n)$, $|(\hat{\Gamma}\hat{\Gamma}^\T\hat{\Sigma})_{uv}| \leq 4\bar{\kappa}\lambda(1+d_\text{max})$. In other words, $|(\hat{\Gamma}\hat{\Gamma}^\T\hat{\Sigma})_{uv}| \to 0$ as $n \to \infty$.

\vspace{0.1in}
\noindent\underline{Setting II.3. Adding $(v, u)$ violates acyclicity but adding $(u,v)$ does not.} 

In this case, even if $(\hat{\Gamma}\hat{\Gamma}^\T\hat{\Sigma})_{uv}$ does not converge to zero, we have by the setting assumption that adding $(u,v)$ to $\hat{E}$ does not violate DAG constraint. Since $\hat{E}$ specifies a DAG, the off-diagonal nonzero entries of the matrix $\hat{\Gamma}\hat{\Gamma}^\T\hat{\Sigma}$ specifies a DAG as well.

\vspace{0.1in}
Putting Settings I--II together, we have shown that as $n \to \infty$, the nonzero entries in the off-diagonal of $\hat{\Gamma}\hat{\Gamma}^\T\hat{\Sigma}$ specify a DAG. Furthermore, according to item iii. of Lemma~\ref{lem:equal_trace}, the diagonal entries of this matrix are equal to one. Since the off-diagonal pattern corresponds to a DAG, we can permute the rows and columns of this matrix to transform it into an upper triangular form without changing its determinant. Given that all diagonal entries are one, the determinant of the matrix is one, and hence $-\log\det(\hat{\Gamma}\hat{\Gamma}^\T\hat{\Sigma}) \to 0$ as $n \to \infty$, and, consequently, $f(\hat{\Gamma})-f(\hat{\Gamma}^\mathrm{opt}) \to 0$.

\begin{figure}[tb]
\begin{minipage}{0.3\textwidth}
\begin{center}
    \begin{tikzpicture}[>=Latex]
        \node[circle, draw] (u) at (0,0) {u};
        \node[circle, draw] (v) at (2,0) {v};
        \node[circle, draw] (i) at (1,-1.5) {i};
        \draw[->] (u) -- (v);
        \draw[->] (v) -- (i);
        \draw[->] (u)  -- node[midway] {\Large $\times$} (i);
    \end{tikzpicture}
\end{center}
\end{minipage}
\begin{minipage}{0.3\textwidth}

    \begin{center}
    \begin{tikzpicture}[>=Latex]
    
        \node[circle, draw] (u) at (0,0) {u};
        \node[circle, draw] (v) at (2,0) {v};
        \node[circle, draw] (i) at (1,-1.5) {i};
        
        \draw[-] (u) -- node[midway] {\Large $\times$}  (v);
        \draw[->] (v) -- node[midway] {\Large $\times$} (i);
        \draw[->] (u) -- (i);
        \draw[->, dashed] (u) to[out=45, in=135] node[midway] {\Large $\times$} (v);
    \end{tikzpicture}
    \end{center}
\end{minipage}
\begin{minipage}{0.3\textwidth}
\begin{center}
    \begin{tikzpicture}[>=Latex]
    
        \node[circle, draw] (u) at (0,0) {u};
        \node[circle, draw] (v) at (2,0) {v};
        \node[circle, draw] (i) at (1,-1.5) {i};
        
        \draw[-] (u) -- node[midway] {\Large $\times$}  (v);
        \draw[->] (v) -- node[midway] {\Large $\times$} (i);
        \draw[->] (u) -- (i);
    \end{tikzpicture}
    \end{center}
\end{minipage}
\caption{Left: scenario for Setting I, middle: scenario for setting II.1, and right: scenario for setting II.2; solid directed edges represent directed edges that are assumed to be in the estimate $\hat{E}$, crossed out solid directed edges represent directed edges that are assumed to be excluded in the estimate $\hat{E}$, crossed out solid undirected edges indicate that the corresponding nodes are not connected in $\hat{E}$, and crossed out dashed directed edge indicates that the edge is not present in $\hat{E}$ as adding it would create a cycle.}
\label{fig:dags_diff_sett}
\end{figure}

\noindent \textbf{Part 2)} When the sample size $n$ is finite, the entries of the nonzero value of $\hat{\Gamma}\hat{\Gamma}^\T\hat{\Sigma}$ that would be zero in the ideal DAG structure are instead small but nonzero, due to estimation error in the sample covariance matrix. We use the matrix $\Delta$ to capture these small perturbations. Then, using the proof of Theorem~\ref{thm:obj} part i), we can immediately conclude that the matrix $\hat{\Gamma}\hat{\Gamma}^\T\hat{\Sigma}_n$ can be decomposed as the sum $N+\Delta$. Here, the off-diagonal entries of $N$ specify a DAG, with ones on the diagonal and under the assumption on $n$, with probability greater than $1-\mathcal{O}(1/n)$, $\|\Delta\|_\infty \leq 4\bar{\kappa}(1+d_\text{max})\lambda$ with zeros on the diagonal of $\Delta$ (here, $\|\cdot\|_\infty$ denotes the maximum entry of the input matrix in absolute value). Consequently, $\|\Delta\|_2 \leq 4m\bar{\kappa}(1+d_\text{max})\lambda$. Furthermore, by Lemma~\ref{lemma:finite_sample}, we get 
$\sigma_\text{min}(\hat{\Gamma}\hat{\Gamma}^\T\hat{\Sigma}_n) \geq \sigma_\text{min}(\hat{\Gamma})^2\sigma_\text{min}(\hat{\Sigma}) \geq \underline{\kappa}^4/4$. The reverse triangle inequality yields $\sigma_\text{min}(N) \geq \underline{\kappa}^4/4- 4m\bar{\kappa}^2(1+d_\text{max})\lambda$. Consider any matrix $\bar{N}$ with $|\bar{N}_{ij}-N_{ij}| \leq |\Delta_{ij}|$. Using the reverse triangle inequality again, we get $\sigma_\text{min}(\bar{N}) \geq \underline{\kappa}^4-8m\bar{\kappa}^2(1+d_\text{max})\lambda$ with probability greater than $1-\mathcal{O}(1/n)$. By the assumption that the sample size satisfies $n/\log(n) \geq \mathcal{O}(m^2\log{m})$, $\bar{N}$ is invertible, and so we can use the first-order Taylor series expansion to obtain
$-\log\det(N+\Delta) = -\log\det(N)-\mathrm{tr}(\bar{N}^{-1}\Delta)$. Since $\log\det(N) = 0$, we obtain the bound $-\log\det(N+\Delta) \leq -\mathrm{tr}(\bar{N}^{-1}\Delta) \leq \|\bar{N}^{-1}\|_2\|\Delta\|_\star$ with $\|\cdot\|_\star$ denoting the nuclear norm. Thus, $-\log\det(N+\Delta) \leq \|\bar{N}^{-1}\|_2\|\Delta\|_\star \leq \frac{m}{\sigma_\text{min}(\bar{N})}\|\Delta\|_2 \leq  \frac{4m^2\bar{\kappa}^2(1+d_\text{max})\lambda}{\underline{\kappa}^4/4-8m\bar{\kappa}^2(1+d_\text{max})\lambda}$.
As $\lambda^2 = \mathcal{O}(\log{m}/n)$, by the assumption on the sample size, $f(\hat{\Gamma})-f(\hat{\Gamma}^\mathrm{opt}) \leq \mathcal{O}(\sqrt{d_\mathrm{max}^2m^4\log{m}/n})$. 

\end{proof}

\subsection{On the Optimization Landscape of Problem \eqref{Problem:micp}}
\label{sec:landscape}
The optimization problem \eqref{Problem:micp} is highly non-convex, due to the $\ell_0$ penalty as well as the DAG constraint. As a result, its landscape may contain many local minimizers. The following corollary to Theorem~\ref{thm:obj} analyzes these local minimizers.
\begin{corollary}Consider the setup in Theorem~\ref{thm:obj}. Let $\tilde{\Gamma}$ be any local minimizer to \eqref{Problem:micp}. Then, $f(\tilde{\Gamma})-f(\hat{\Gamma}^{\mathrm{opt}}) \to_P 0$ as $n\rightarrow \infty$. Furthermore, if $n/\log(n) \geq \mathcal{O}(m^2\log{m})$, with probability greater than $1-1/\mathcal{O}(n)$, we have 
$0 \leq f(\tilde{\Gamma}) - f(\hat{\Gamma}^{\mathrm{opt}}) \leq \mathcal{O}(\sqrt{d_{\text{max}}^2m^4\log{m}/n})$.
\label{corr:local}
\end{corollary}
This corollary follows from Theorem~\ref{thm:obj} and the fact that coordinate-wise optimality is weaker than local optimality; therefore, any local minima of this problem satisfies the same result.  This result implies that when the sample size is large, all local minima achieve objective values close to a global minimum; moreover, in the infinite sample regime, all local minimizers are global minimizers. Given the non-convexity of our objective function, this result offers significant insight into the optimization landscape of the problem.

\section{Selecting a Suitable Ordering for Coordinate Descent Updates}
\label{sec:selecting_ordering}
A significant challenge for nonconvex optimization problems is that selecting a good starting point and/or an effective coordinate update order is crucial for achieving good solutions. For example, in our setting, the performance of Algorithm~\ref{algo:cd_spacer} can be sensitive to the input ordering $O$ for coordinate descent updates. To address this challenge, we propose estimating a topological ordering of the underlying DAG $\mathcal{G}^\star$ and using it as the input ordering $O$. A valid topological ordering of $\mathcal{G}^\star$ is a permutation $(a_1,a_2,\dots,a_m)$ of $(1,2,\dots,m)$ such that for every edge $i \to j$ in $\mathcal{G}^\star$, the index $a_i$ precedes $a_j$ in the ordering.

To estimate a good topological ordering, we use an algorithm proposed by \citet{Chen19}. This procedure, outlined in Algorithm~\ref{algo:TO}, iteratively selects the node with the smallest conditional variance given the variables selected in the previous step.  When the variances of the noise (i.e., variance of the coordinates of $\epsilon$ in \eqref{eqn:sem}) are identical, they show that with high probability, their algorithm correctly recovers a valid topological ordering of $\mathcal{G}^\star$. 

We improve upon the theoretical results of \citet{Chen19} and show that we can still use Algorithm~\ref{algo:TO} to recover a valid topological ordering, even when the noise variances are not identical. Our result assumes that the noise variances are not too far apart---that is, the noise distribution is nearly homoscedastic.

\begin{algorithm}
\caption{Topological Ordering \citep{Chen19}}
    \begin{algorithmic}[1]
        \STATE \textbf{Input: } $\hat{\Sigma} \in \mathbb{R}^{p \times p}$ (estimated) covariance of $X$
        \STATE $C \gets \emptyset$\
        \FOR{$t = 1, \ldots, m$}
            \STATE $j^* \gets \arg\min\limits_{j \in V \setminus C} \hat{\Sigma}_{j,j} - \hat{\Sigma}_{j,C}\hat{\Sigma}_{C,C}^{-1}\hat{\Sigma}_{C,j} = \frac{1}{\left\{ (\hat{\Sigma}_{C \cup \{j\}, C \cup \{j\}})^{-1} \right\}_{j,j}}$\
            \STATE Append $j^*$ to $C$\
        \ENDFOR
        \STATE \textbf{Output: }  $C$
    \end{algorithmic}
    \label{algo:TO}
\end{algorithm}

\begin{assumption}[Near homoscedastic noise] Define $\zeta := \min_{(j,k)\in E} (B^{\star}_{jk})^2$. Let $\Omega^\star_{\max}:=\max_{j}\Omega^\star_{j,j}$ and $\Omega^\star_{\min}:=\min_{j}\Omega^\star_{j,j}$ denote the maximum and minimum error variance, respectively. We assume $\Omega^\star_{\max} < \Omega^\star_{\min}(1+\zeta)$; that is, the maximum and minimum error variances are not too far apart.
\label{ass:near_homoscedastic}
\end{assumption}

\begin{theorem}
\label{thm:ordering}
Let $X$ be variables defined in \eqref{eqn:sem}. Let $\lambda_{\min}>0$ be the smallest eigenvalue of the population covariance matrix $\Sigma^\star$. Suppose Assumption~\ref{ass:near_homoscedastic} holds. If

\begin{align*}
n >\; & 128m^2 \log \left( \frac{2m^2 + 2m}{\epsilon} \right)\left( 1 + 4 \frac{\Omega^\star_{\max}}{\Omega^\star_{\min}} \right)^2 \left( \max_{j \in V} \Sigma^\star_{j,j} \right)^2 \\& \times \left(\frac{1}{\lambda_{\min}} + \frac{2\Omega^\star_{\max}\Omega^\star_{\min}(1+\zeta)}{(\Omega^\star_{\min}(1+\zeta) - \Omega^\star_{\max})\lambda^2_{\min}} \right)^2,
\end{align*}
then Algorithm~\ref{algo:TO} recovers a valid topological ordering of $\mathcal{G}^\star$ with probability at least \( 1 - \epsilon \).
\end{theorem}
The proof of Theorem~\ref{thm:ordering}, presented in Appendix~\ref{proof:ordering}, adapts the argument from \citet{Chen19}. Theorem~\ref{thm:ordering} shows that, under Assumption~\ref{ass:near_homoscedastic} and with a sufficiently large sample size, Algorithm~\ref{algo:TO} recovers the correct topological ordering with high probability. Unlike the result in \citet{Chen19}, which assumes equal noise variances (homoscedasticity), our analysis allows for heteroscedastic noise, provided the variances are not too dissimilar, as formalized in Assumption~\ref{ass:near_homoscedastic}. This generalization is important in practice, where the assumption of exactly equal noise variances can be unrealistic. 

\begin{remark}[Statistical consistency.]
Suppose a valid topological ordering of the $m$ variables is given.  Then, under such ordering, the true weighted adjacency matrix $\Gamma^\star$ of the DAG is strictly upper‐triangular, and the population covariance satisfies $\Sigma^\star = {\Gamma^\star}^{\invT}\,{\Gamma^\star}^{-1}$. This implies that $\Gamma^\star = {L^\star}^\invT$, where $L^\star $ is the Cholesky factor of $\Sigma^\star$, i.e., $\Sigma^\star = L^\star {L^\star}^\T$. Thus, the problem of finding the DAG reduces to the known sparse Cholesky factor estimation problem. In fact, consistency (in Frobenius or operator norm) of the sparse Cholesky factor follows directly from existing results on regularized Cholesky‐factor estimation; see, for example, Theorem~10 of \citet{lam2009sparsistency}. 

\end{remark}

\section{Synthetic and Real Experiments}
\label{sec:experiments}
In this section, we illustrate the utility of our method in Algorithm~\ref{algo:cd_spacer} on synthetic and real data and compare its performance with competing methods. We dub our method CD-$\ell_0$ as it is a coordinate descent method using $\ell_0$ penalized loss function. The competing methods we compare against include greedy equivalence search (GES) \citep{chickering2002optimal}, NOTEARS \citep{zheng2018dags}, and the mixed-integer convex program (MICODAG) \citep{xu2024integer}. We also compare our method with other coordinate descent algorithms (CCDr-MCP) \citep{aragam2015concave, aragam2019learning, Fu13}, which use a minimax concave penalty instead of $\ell_0$ norm and are implemented as an \texttt{R} package \textit{sparsebn}. To obtain an input ordering $O$ for CD-$\ell_0$, we use the algorithm of \citet{Chen19} (presented in Algorithm~\ref{algo:TO}). Consequently, we include in our comparisons the method of \citet{Chen19}, which performs statistical testing to infer a network structure using this estimated ordering. All experiments were conducted on an Intel Xeon Gold 6230R CPU with eight cores and 8 GB of memory, using \texttt{Gurobi} 10.0.1 as the optimization solver.

As the input super-structure $E_\text{super}$, we supply an estimated moral graph, computed using the graphical lasso procedure \citep{friedman2008sparse}. In all experiments, we set the graphical lasso penalty to $0.01$ (using a larger value for some datasets to avoid ill-conditioning) and then threshold the estimated precision matrix by retaining only those entries with absolute value $\geq 0.1$. To make our comparisons fair, we appropriately modify the competing methods so that $E_\text{super}$ can also be supplied as input.  Note that we count the number of support after each update in Algorithm \ref{algo:cd_spacer}. Converting the graph into a string key at each iteration is inefficient. Therefore, in the implementation, we count the support only after each full loop, setting the threshold to $C$ instead of $Cm^2$. Throughout this paper, $C$ is set to 5.
We experimented with different initializations of $\Gamma^{\mathrm{init}}$ beyond the default value and observed similar performance across the various choices. As a result, throughout this section, we use the default $\Gamma^{\mathrm{init}} = \text{identity}$. We emphasize again that our theoretical convergence and optimality results in Section~\ref{sec:theory} hold for any initialization.

We use the metric $d_\text{cpdag}$ to evaluate the estimation accuracy as the underlying DAG is generally identifiable up to the Markov equivalence class. The metric $d_\text{cpdag}$ is the number of different entries between the unweighted adjacency matrices of the estimated completed partially directed acyclic graph (CPDAG) and the true CPDAG. A CPDAG has a directed edge from a node $i$ to a node $j$ if and only if this directed edge is present in every DAG in the associated Markov equivalence class, and it has an undirected edge between nodes $i$ and $j$ if the corresponding Markov equivalence class contains DAGs with both directed edges from $i$ to $j$ and from $j$ to $i$. 

The time limit for the integer programming method MICODAG is set to $50m$. If the algorithm does not terminate within the time limit, we report the solution time (in seconds) and the achieved relative optimality gap, computed as $\mathrm{RGAP}=  (\text{upper bound} - \text{lower bound})/\text{lower bound}$. Here, the $\text{upper bound}$ and $\text{lower bound}$ refer to the objective value associated with the best feasible solution and best lower bound, obtained respectively by MICODAG. A zero value for $\mathrm{RGAP}$ indicates that an optimal solution has been found. 

For each method, we report results under oracle tuning, i.e., selecting the parameter that yields the best performance. This provides a benchmark of the best-case potential of the methods.

\emph{Setup of synthetic experiments}: For all the synthetic experiments, following the experimental setup in \cite{xu2024integer}, once we specify a DAG, we generate data according to the SEM \eqref{eqn:sem}, where the nonzero entries of $B^\star$ are drawn uniformly at random from the set $\{-0.8, -0.6, 0.6, 0.8\}$. Unless otherwise specified, the diagonal entries of $\Omega^\star$ are drawn uniformly at random from the set $\{0.8, 1, 1.2\}$. This generation of noise variance introduces some heteroscedasticity. Later, we vary this set to study how our method performs under different levels of heteroscedasticity.

\subsection{Convergence of CD-\texorpdfstring{$\ell_0$}{l0} solution to an optimal solution} Theorem~\ref{thm:obj} states that as the sample size tends to infinity, CD-$\ell_0$ identifies an optimally scoring model. To observe how quickly the asymptotics take effect, we generate three synthetic DAGs with $m=10$ nodes, where the total number of edges is chosen from the set $\{7, 12, 21\}$. We obtain $10$ independently and identically distributed data sets according to the SEM described earlier with sample size $n=\{100, 200, 300, 400, 500, 800, 1000, 1600, 3200\}$. 
In Figure~\ref{fig:obj_diff}, we compute the normalized difference 
$(\mathrm{obj}^{\mathrm{method}} - \mathrm{obj}^{\mathrm{opt}})/\mathrm{obj}^{\mathrm{opt}}$ as a function of $n$ for the three graphs, averaged across the ten independent trials. Here, $\mathrm{obj}^{\mathrm{method}}$ is the objective value obtained by the corresponding method  (CD-$\ell_0$ or GES), while $\mathrm{obj}^{\mathrm{opt}}$ is the optimal objective of \eqref{Problem:micp} obtained by the integer programming approach MICODAG. For moderately large sample sizes, CD-$\ell_0$ attains the optimal objective value, whereas GES does not. Meanwhile, CD-$\ell_0$ obtains a near-optimal solution at about 0.1\% of time, as compared to MICODAG. In the left panel of Figure~\ref{fig:obj_diff}, we run CD-$\ell_0$ with three different random orderings. In all cases, the objective converges to zero as expected, though at varying rates. The worse-case rate is $\mathcal{O}(\sqrt{d_{\mathrm{max}}^2m^4\log m/n})$ according to Theorem~\ref{thm:obj}.

\begin{figure}[tb]
    \centering
    \begin{subfigure}[t]{0.32\textwidth}
        \centering
        \includegraphics[scale=0.3]{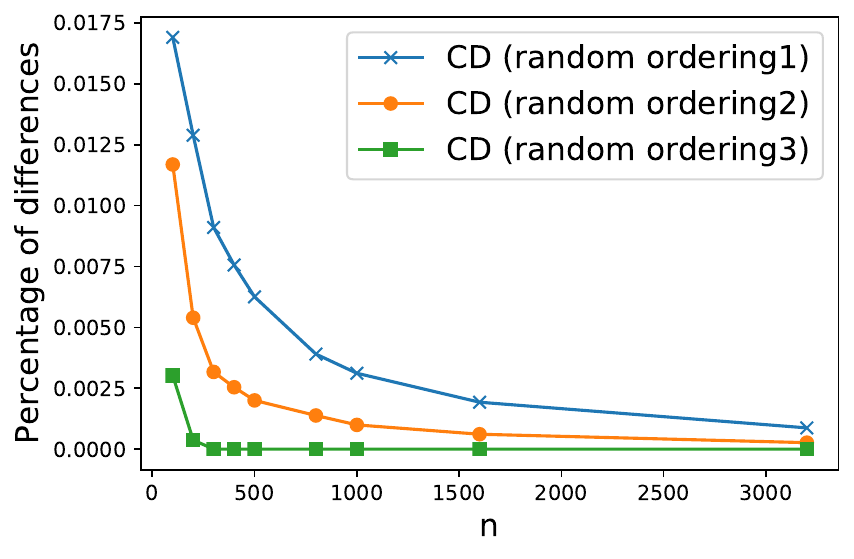}
    \end{subfigure}\hspace{-0.2cm}
    \begin{subfigure}[t]{0.32\textwidth}
        \centering
        \includegraphics[scale=0.3]{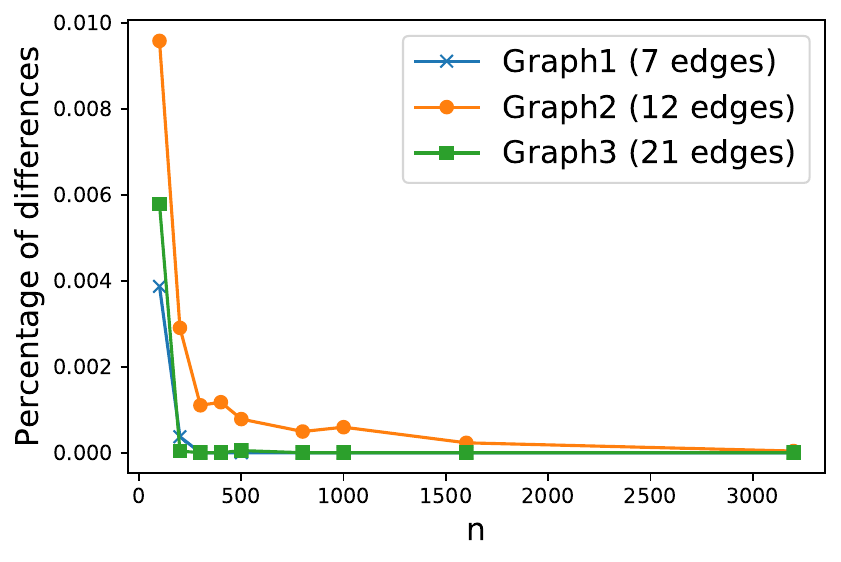}
    \end{subfigure}\hspace{-0.2cm}
    \begin{subfigure}[t]{0.32\textwidth}
        \centering
        \includegraphics[scale=0.3]{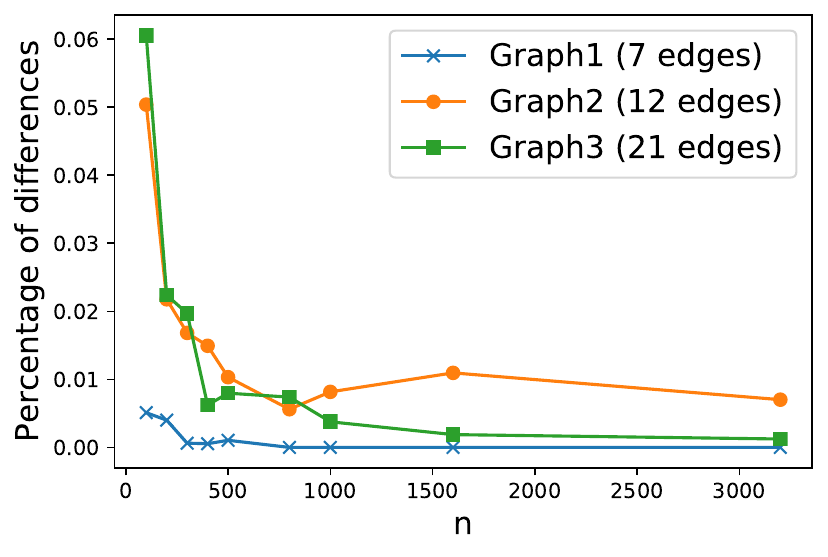}
    \end{subfigure}
    
    \caption{Convergence of CD-$\ell_0$ to an optimal solution}
    {\raggedright \footnotesize Left: normalized difference of CD-$\ell_0$ with 3 different random orderings; Middle: normalized difference for different graphs; Right: normalized difference of objectives of solutions obtained from MICODAG and GES. All results are computed and averaged over ten independent trials. 
 \par}
    \label{fig:obj_diff}
\end{figure}

\subsection{Comparison to benchmarks under near-homoscedastic error}

We next generate data sets from twelve publicly available networks sourced from  \cite{Manzour21} and the Bayesian Network Repository (bnlearn). These networks have different numbers of nodes, ranging from $m = 6$ to $m = 70$. We generate 10 independently and identically distributed data sets for each network according to the SEM described earlier with sample size $n=500$. Since the diagonal entries of $\Omega^\star$ are close (with the largest deviation being $0.8$ to $1.2$), our algorithm performs very similarly to TD, since the topological ordering obtained by Algorithm~\ref{algo:TO} is likely to be correct in this near-homoscedastic case, and both methods perform well. For this reason, we present the results for the nearly homoscedastic case in Appendix~\ref{sec:extra_experiment1}, and instead focus on a slightly higher degree of heteroscedasticity to highlight that our method is more robust to heteroscedasticity than TD. To that end, the diagonal entries of $\Omega^\star$ are drawn uniformly from $\{0.6, 1, 1.2\}$.

Table \ref{tab:compare_benchmarks_est} compares the performance of our method CD-$\ell_0$ with the competing ones. First, consider small graphs ($m \leq 20$) for which the integer programming approach MICODAG achieves an optimal or near-optimal solution with a small RGAP. As expected, in terms of the accuracy of the estimated model, MICODAG exhibits the best performance. For these small graphs, CD-$\ell_0$ performs similarly to MICODAG but attains the solutions much faster. Next, consider moderately sized graphs ($m > 20$). In this case, MICODAG cannot solve these problem instances within the time limit and hence finds inaccurate models, whereas CD-$\ell_0$ obtains much more accurate models much faster. Finally, CD-$\ell_0$ outperforms CCDr-MCP, GES, TD, and NOTEARS in most problem instances. The improved performance of CD-$\ell_0$ over CCDr-MCP highlights the advantage of using $\ell_0$ penalization over a minimax concave penalty: $\ell_0$ penalization ensures that DAGs in the same Markov equivalence class have the same score. In contrast, this property does not hold for other penalties.  Our method outperforms NOTEARS in terms of the $d_{\mathrm{cpdag}}$ for most graphs. Both methods are designed for computational efficiency, but NOTEARS uses a least squares loss, which implicitly assumes homoscedastic (equal variance) noise \citep{vgBuhlmann}. In contrast, our approach directly optimizes the original likelihood function, allowing us to account for heteroscedasticity. This results in improved accuracy and robustness under varying noise conditions.

\begin{table}[tb]
\def~{\hphantom{0}}
    \centering
    \caption{Comparison of our method, CD-$\ell_0$, with competing methods}
    \resizebox{\columnwidth}{!}{\begin{tabular}{lccccccccccccc}

        & \multicolumn{3}{c}{$\mathrm{MICODAG}$} & \multicolumn{2}{c}{CCDr-MCP} & \multicolumn{2}{c}{GES} & \multicolumn{2}{c}{TD} & \multicolumn{2}{c}{NOTEARS} & \multicolumn{2}{c}{CD-$\ell_0$} \\

        Network($m$) & Time & \rgap & $d_{\mathrm{cpdag}}$ & Time & $d_{\mathrm{cpdag}}$ & Time & $d_{\mathrm{cpdag}}$ & Time & $d_{\mathrm{cpdag}}$ & Time & $d_{\mathrm{cpdag}}$ & Time & $d_{\mathrm{cpdag}}$ \\

        Dsep(6)       & $\leq 1$   & 0     & 2.0$(\pm0)$    
                      & $\leq 1$   & 2.0$(\pm0)$    
                      & $\leq 1$   & 2.0$(\pm0)$    
                      & $\leq 1$   & 2.1$(\pm0.3)$    
                      & $\leq 1$   & 2.0$(\pm0.0)$    
                      & $\leq 1$   & 2.0$(\pm0)$ \\

        Asia(8)       & $\leq 1$        & 0     & 2.0$(\pm0)$    
                      & $\leq 1$   & 2.3$(\pm0.9)$    
                      & $\leq 1$   & 2.1$(\pm0.3)$    
                      & $\leq 1$   & 2.1$(\pm0.3)$    
                      & $\leq 1$   & 3.5$(\pm2.4)$    
                      & $\leq 1$   & 2.1$(\pm0.3)$ \\

        Bowling(9)    & 1.2        & 0     & 2.0$(\pm0)$    
                      & $\leq 1$   & 3.1$(\pm1.9)$    
                      & $\leq 1$   & 11.8$(\pm8.0)$   
                      & $\leq 1$   & 2.2$(\pm0.6)$    
                      & $\leq 1$   & 3.6$(\pm0.8)$    
                      & $\leq 1$   & 2.0$(\pm0)$ \\

        InsSmall(15)  & $\geq 750$ & .007  & 7.2$(\pm1.0)$   
                      & $\leq 1$   & 27.0$(\pm4.1)$   
                      & $\leq 1$   & 37.2$(\pm4.7)$   
                      & $\leq 1$   & 3.9$(\pm4.8)$    
                      & $\leq 1$   & 10.9$(\pm1.6)$   
                      & $\leq 1$   & 4.0$(\pm4.6)$ \\

        Rain(14)      & 56.0       & 0     & 2.0$(\pm0)$    
                      & $\leq 1$   & 8.5$(\pm2.1)$    
                      & $\leq 1$   & 15.5$(\pm2.2)$   
                      & $\leq 1$   & 6.6$(\pm2.1)$    
                      & $\leq 1$   & 5.4$(\pm2.1)$    
                      & $\leq 1$   & 5.5$(\pm2.0)$ \\

        Cloud(16)     & 9.7       & 0     & 12.7$(\pm5.5)$    
                      & $\leq 1$   & 8.5$(\pm4.5)$    
                      & $\leq 1$   & 10.6$(\pm8.2)$   
                      & $\leq 1$   & 11.9$(\pm7.8)$   
                      & $\leq 1$   & 17.6$(\pm3.7)$   
                      & $\leq 1$   & 11.6$(\pm6.5)$ \\

        Funnel(18)    & 8.0       & 0     & 2.0$(\pm0)$    
                      & $\leq 1$   & 3.3$(\pm1.7)$    
                      & $\leq 1$   & 2.6$(\pm1.3)$    
                      & $\leq 1$   & 6.4$(\pm2.5)$    
                      & $\leq 1$   & 2.9$(\pm1.4)$    
                      & $\leq 1$   & 5.4$(\pm2.0)$ \\

        Galaxy(20)    & 48.6      & 0     & 1.0$(\pm0)$    
                      & $\leq 1$   & 9.1$(\pm4.3)$    
                      & $\leq 1$   & 26.2$(\pm0.6)$   
                      & $\leq 1$   & 11.7$(\pm6.7)$   
                      & $\leq 1$   & 12.4$(\pm2.1)$   
                      & $\leq 1$   & 9.7$(\pm4.9)$ \\

        Insurance(27) & $\geq 1350$& .272  &  18.5$(\pm8.5)$  
                      & $\leq 1$   & 32.5$(\pm6.0)$   
                      & $\leq 1$   & 24.2$(\pm7.9)$   
                      & $\leq 1$   & 16.8$(\pm2.2)$   
                      & $\leq 1$   & 12.8$(\pm4.0)$   
                      & $\leq 1$   & 18.3$(\pm1.9)$ \\

        Factors(27)   & $\geq 1350$& .242  &  60.4$(\pm7.3)$  
                      & $\leq 1$   & 65.1$(\pm6.4)$   
                      & $\leq 1$   & 68.5$(\pm9.3)$   
                      & $\leq 1$   & 26.9$(\pm6.4)$   
                      & $\leq 1$   & 41.7$(\pm5.1)$   
                      & $\leq 1$   & 25.2$(\pm9.7)$ \\

        Hfinder(56)   & $\geq 2800$& .181  & 12.7$(\pm4.6)$   
                      & $\leq 1$   & 14.9$(\pm3.8)$   
                      & $\leq 1$   & 27.5$(\pm14.4)$  
                      & 1.2        & 38.3$(\pm7.7)$   
                      & 1.1& 17.3$(\pm4.0)$   
                      & $\leq 1$   & 45.1$(\pm8.9)$ \\

        Hepar(70)     & $\geq 3500$&  .868 & 45.6$(\pm9.8)$     
                      & $\leq 1$   & 51.6$(\pm9.5)$   
                      & $\leq 1$   & 71.9$(\pm19.7)$  
                      & 2.6        & 54.2$(\pm9.8)$   
                      & 2.2& 43.9$(\pm5.4)$   
                      & $\leq 1$   & 38.5$(\pm5.9)$ \\
    \end{tabular}}
    {\raggedright \footnotesize {Here, MICODAG, mixed-integer convex program \citep{xu2024integer}; CCDr-MCP, minimax concave penalized estimator with coordinate descent \citep{aragam2019learning}; GES, greedy equivalence search algorithm \citep{chickering2002optimal}; TD, top-down method \citep{Chen19}; NOTEARS, \citep{zheng2018dags}; $d_{\mathrm{cpdag}}$, differences between the true and estimated completed partially directed acyclic graphs; \rgap, relative optimality gap. All results are computed over ten independent trials, where the average $d_\mathrm{cpdag}$ values are presented with their standard deviations.}\par}
    \label{tab:compare_benchmarks_est}
\end{table}

\textbf{Large graphs}: We next demonstrate the scalability of our coordinate descent algorithm for learning large DAGs with over $100$ nodes. We consider networks from the Bayesian Network Repository and generate $10$ independent datasets similar to those in the previous experiment. The method TD is excluded since it cannot scale to very large graphs. Table~\ref{tab:compare_benchmarks_large_est} presents the results where we see that our method CD-$\ell_0$ can effectively scale to large graphs and obtain better or comparable performance to competing methods, as measured by the $d_{\mathrm{cpdag}}$ metric.

\begin{table}[tb]
\def~{\hphantom{0}}
    \centering
    \caption{Comparison of our method, CD-$\ell_0$, with competing methods for large graphs}
    \resizebox{\columnwidth}{!}{
    \begin{tabular}{lcccccc}
        &  \multicolumn{2}{c}{CCDr-MCP}  &  \multicolumn{2}{c}{GES}  &  \multicolumn{2}{c}{CD-$\ell_0$}\\
       Network($m$) &   Time & $d_{\mathrm{cpdag}}$  &   Time & $d_{\mathrm{cpdag}}$  &   Time & $d_{\mathrm{cpdag}}$\\
       Pathfinder(109) & $5.5$ & 206.6$(\pm4.86)$     & $\leq 1$ & 253.8$(\pm19.52)$     & 24.0  & 95.0$(\pm18.40)$      \\
        Andes(223)      & $12.6$ & 121.5$(\pm5.23)$      & $1.9$ & 152.8$(\pm22.26)$     & 108.9  & 98.4$(\pm6.82)$      \\
        Diabetes(413)   & $42.6$ & 322.9$(\pm9.39)$      & $10.7$ & 458.7$(\pm28.73)$     & 732.3 & 158.4$(\pm10.76)$ 
    \end{tabular}
    }
      {\raggedright \footnotesize See Table~\ref{tab:compare_benchmarks_est} for the description of the methods. All results are computed over ten independent trials, where the average $d_\mathrm{cpdag}$ values are presented with their standard deviations. \par}
\label{tab:compare_benchmarks_large_est}
\end{table}

\subsection{Severe heteroscedastic error and limitations of our method}

One limitation of our algorithm is that under severe heteroscedasticity, recovering an accurate topological ordering becomes difficult, which can degrade performance. To demonstrate this, we draw diagonal entries of $\Omega^\star$ uniformly at random from the set $\{0.2,1,5\}$, so that the noise distribution is highly heteroscedastic.  In this setup, Assumption~\ref{ass:near_homoscedastic} is unlikely to hold, making the input topological ordering less reliable for both TD and our method. Despite this, the results in Table~\ref{tab:comparison_reorder_extra_large} show that both approaches remain competitive. 

\begin{table}[h!]
\centering
\caption{Comparison of GES, Top-down (TD), and CD-$\ell_0$ using datasets with very large heteroscedasticity.}
\begin{tabular}{lccc}
&\multicolumn{3}{c}{$d_{\mathrm{cpdag}}$}\\
Network($m$) &GES & TD & $\text{CD-}{\ell_0}$ \\
Dsep(6) & $4.9(\pm1.4)$ & $5.4(\pm0.8)$ & $4.7(\pm1.2)$ \\
Asia(8) & $6.4(\pm2.8)$ & $14.4(\pm1.6)$ & $16.4(\pm1.0)$ \\
Bowling(9) & $8.2(\pm6.8)$ & $10.0(\pm0)$ & $10.0(\pm0)$ \\
InsuranceSmall(15) & $22.5(\pm5.1)$ & $18.7(\pm2.5)$ & $17.2(\pm1.6)$ \\
Rain(14) & $26.6(\pm1.8)$ & $10.5(\pm2.8)$ & $9.9(\pm2.3)$ \\
Cloud(16) & $14.9(\pm3.6)$ & $19.4(\pm1.8)$ & $19.0(\pm3.8)$ \\
Funnel(18) & $10.0(\pm2.1)$ & $9.4(\pm2.1)$ & $7.6(\pm1.4)$ \\
Galaxy(20) & $6.0(\pm2.4)$ & $35.6(\pm3.9)$ & $26.5(\pm1.6)$ \\
Insurance(27) & $61.1(\pm7.7)$ & $65.8(\pm3.6)$ & $46.9(\pm4.1)$ \\
Factors(27) & $74.1(\pm4.8)$ & $47.5(\pm7.5)$ & $40.0(\pm10.0)$ \\
Hfinder(56) & $41.0(\pm5.1)$ & $93.6(\pm2.8)$ & $100.0(\pm3.7)$ \\
Hepar(70) & $68.5(\pm11.9)$ & $124.1(\pm8.7)$ & $91.5(\pm12.5)$ \\
\end{tabular}
\label{tab:comparison_reorder_extra_large}
\end{table}

Our algorithm CD-$\ell_0$ does not rely on a specific topological ordering and can take any ordering as input---not just the one produced by TD. In practice, if an accurate method for obtaining the topological ordering in a heteroscedastic setting is available, it can be used as input to Algorithm~\ref{algo:cd_spacer}.

\subsection{Real data from causal chambers}
Recently, \citet{gamella2024causal} constructed two devices, referred to as causal chambers, allowing us to quickly and inexpensively produce large data sets from non-trivial but well-understood real physical systems. The ground-truth DAG underlying this system is known and shown in Figure~\ref{fig:chamber}(a). We collect $n=1000$ to $n =10000$ observational samples of $m=20$ variables at increments of $1000$. To maintain clarity, we only plot a subset of the variables in Figure \ref{fig:chamber}(a, b, c). However, the analysis includes all variables. With this data, we obtain estimates for the Markov equivalence class of the ground-truth DAG using GES and our method CD-$\ell_0$ and measure the accuracy of the estimates using the $d_\mathrm{cpdag}$ metric.

\begin{figure}[ht]
    \centering
\begin{subfigure}{0.23\textwidth}
\centering
\begin{tikzpicture}[scale=0.65, transform shape]
        \node[circle, inner sep=0.12em] (0) at (-0.919, 1.430) {$R$};
        \node[circle, inner sep=0.12em] (1) at (-1.546, 0.706) {$G$};
        \node[circle, inner sep=0.12em] (2) at (-1.683, -0.242) {$B$};
        \node[circle, inner sep=0.12em] (3) at (-1.285, -1.113) {$\tilde{C}$};
        \node[circle, inner sep=0.12em] (4) at (-0.479, -1.631) {$\tilde{I}_1$};
        \node[circle, inner sep=0.12em] (5) at (0.479, -1.631) {$\tilde{I}_2$};
        \node[circle, inner sep=0.12em] (6) at (1.285, -1.113) {$\tilde{I}_3$};
        \node[circle, inner sep=0.12em] (7) at (1.683, -0.242) {$\theta_1$};
        \node[circle, inner sep=0.12em] (8) at (1.546, 0.706) {$\theta_2$};
        \node[circle, inner sep=0.12em] (9) at (0.919, 1.430) {$\tilde{\theta}_1$};
        \node[circle, inner sep=0.12em] (10) at (0.000, 1.700) {$\tilde{\theta}_2$};
        \begin{scope}[]
            \draw[->] (0) edge (3);
            \draw[->] (0) edge (4);
            \draw[->] (0) edge (5);
            \draw[->] (0) edge (6);
            \draw[->] (1) edge (3);
            \draw[->] (1) edge (4);
            \draw[->] (1) edge (5);
            \draw[->] (1) edge (6);
            \draw[->] (2) edge (3);
            \draw[->] (2) edge (4);
            \draw[->] (2) edge (5);
            \draw[->] (2) edge (6);
            \draw[->] (7) edge (6);
            \draw[->] (7) edge (9);
            \draw[->] (8) edge (6);
            \draw[->] (8) edge (10);
            \draw[->] (9) edge (7);
            \draw[->] (10) edge (8);
        \end{scope}
    \end{tikzpicture}
\caption{Ground truth}
\end{subfigure}
\begin{subfigure}{0.23\textwidth}
\centering
    \begin{tikzpicture}[scale=0.65, transform shape]
        \node[circle, inner sep=0.12em] (0) at (-0.919, 1.430) {$R$};
        \node[circle, inner sep=0.12em] (1) at (-1.546, 0.706) {$G$};
        \node[circle, inner sep=0.12em] (2) at (-1.683, -0.242) {$B$};
        \node[circle, inner sep=0.12em] (3) at (-1.285, -1.113) {$\tilde{C}$};
        \node[circle, inner sep=0.12em] (4) at (-0.479, -1.631) {$\tilde{I}_1$};
        \node[circle, inner sep=0.12em] (5) at (0.479, -1.631) {$\tilde{I}_2$};
        \node[circle, inner sep=0.12em] (6) at (1.285, -1.113) {$\tilde{I}_3$};
        \node[circle, inner sep=0.12em] (7) at (1.683, -0.242) {$\theta_1$};
        \node[circle, inner sep=0.12em] (8) at (1.546, 0.706) {$\theta_2$};
        \node[circle, inner sep=0.12em] (9) at (0.919, 1.430) {$\tilde{\theta}_1$};
        \node[circle, inner sep=0.12em] (10) at (0.000, 1.700) {$\tilde{\theta}_2$};
        \begin{scope}[]
            \draw[->] (0) edge (3);
            \draw[->] (0) edge (4);
            \draw[->] (0) edge (5);
            \draw[->] (0) edge (6);
            \draw[->] (1) edge (3);
            \draw[->] (1) edge (4);
            \draw[->] (1) edge (5);
            \draw[->] (1) edge (6);
            \draw[->] (2) edge (3);
            \draw[->] (2) edge (4);
            \draw[->] (2) edge (5);
            \draw[->] (2) edge (6);
            \draw[dashed, ->, red] (3) edge (5);
            \draw[dashed, ->, red] (5) edge (4);
            \draw[dashed, ->, red] (5) edge (6);
            \draw[->] (7) edge (6);
            \draw[->] (7) edge (9);
            \draw[->] (8) edge (10);
            \draw[->] (9) edge (7);
            \draw[dashed, ->, red] (10) edge (6);
            \draw[->] (10) edge (8);
        \end{scope}
    \end{tikzpicture}
    \caption{GES}
\end{subfigure}
\begin{subfigure}{0.23\textwidth}
\centering
\begin{tikzpicture}[scale=0.65, transform shape]
        \node[circle, inner sep=0.12em] (0) at (-0.919, 1.430) {$R$};
        \node[circle, inner sep=0.12em] (1) at (-1.546, 0.706) {$G$};
        \node[circle, inner sep=0.12em] (2) at (-1.683, -0.242) {$B$};
        \node[circle, inner sep=0.12em] (3) at (-1.285, -1.113) {$\tilde{C}$};
        \node[circle, inner sep=0.12em] (4) at (-0.479, -1.631) {$\tilde{I}_1$};
        \node[circle, inner sep=0.12em] (5) at (0.479, -1.631) {$\tilde{I}_2$};
        \node[circle, inner sep=0.12em] (6) at (1.285, -1.113) {$\tilde{I}_3$};
        \node[circle, inner sep=0.12em] (7) at (1.683, -0.242) {$\theta_1$};
        \node[circle, inner sep=0.12em] (8) at (1.546, 0.706) {$\theta_2$};
        \node[circle, inner sep=0.12em] (9) at (0.919, 1.430) {$\tilde{\theta}_1$};
        \node[circle, inner sep=0.12em] (10) at (0.000, 1.700) {$\tilde{\theta}_2$};
        \begin{scope}[]
            \draw[-] (0) edge (3);
            \draw[->] (0) edge (4);
            \draw[->] (0) edge (5);
            \draw[->] (0) edge (6);
            \draw[-] (1) edge (3);
            \draw[->] (1) edge (4);
            \draw[->] (1) edge (5);
            \draw[->] (1) edge (6);
            \draw[-] (2) edge (3);
            \draw[->] (2) edge (4);
            \draw[->] (2) edge (5);
            \draw[->] (2) edge (6);
            \draw[->] (7) edge (9);
            \draw[->] (8) edge (10);
            \draw[->] (9) edge (7);
            \draw[->] (10) edge (8);
        \end{scope}
    \end{tikzpicture}
   \caption{CD-$\ell_0$}
\end{subfigure}
 \begin{subfigure}{0.28\textwidth}
\centering
    \includegraphics[scale=0.25]{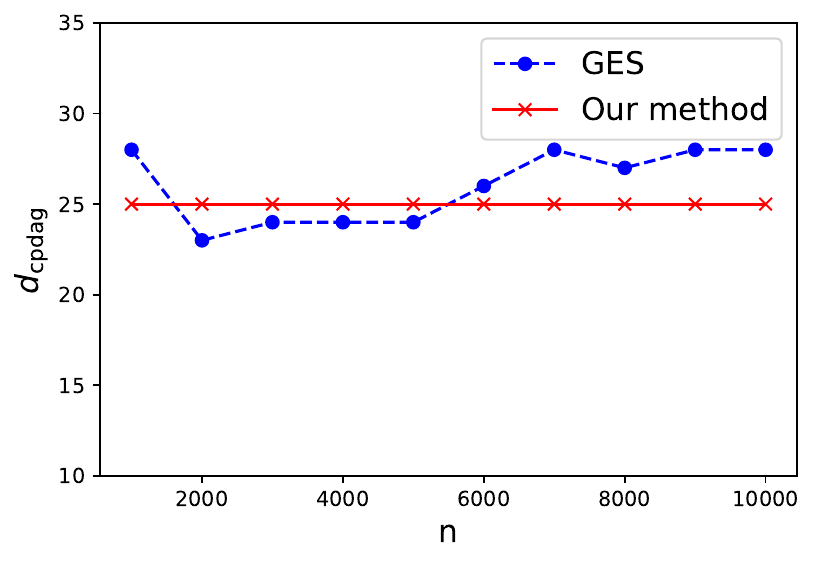}
    \caption{$d_{\mathrm{cpdag}}$ comparison}
    \label{fig:chamber_n}
\end{subfigure}   
\caption{Learning causal models from causal chambers data in \cite{gamella2024causal}}
    {\raggedright \footnotesize Here, a. ground-truth DAG described in \cite{gamella2024causal}, b-c. the estimated CPDAGs by GES and CD-$\ell_0$  for sample size $n = 10000$, d. comparing the accuracy of the CPDAGs estimated by our method CD-$\ell_0$ and GES with different sample sizes $n$; here the accuracy is computed relative to CPDAG of the ground-truth DAG and uses the metric $d_\mathrm{cpdag}$.
    \par} 
\label{fig:chamber}
\end{figure}
Figures \ref{fig:chamber}(b-c) show the estimated CPDAG for each approach when $n = 10000$. Neither method picks up the edges between the polarizer angles $\theta_1$ and $\theta_2$ and other variables. As mentioned in  \cite{gamella2024causal}, this phenomenon is likely due to these effects being nonlinear. However, with a large sample size, CD-$\ell_0$ produces graphs with fewer false positive edges. Figure \ref{fig:chamber}(d) compares the accuracy of CD-$\ell_0$ and GES in estimating the Markov equivalence class of the ground-truth DAG for different sample sizes. We observe that CD-$\ell_0$ and GES perform similarly on average.

\section{Discussion}
\label{sec:discussion}
In this paper, we propose the first coordinate descent procedure with proven optimality guarantees in the context of learning Bayesian networks. 
 Numerical experiments demonstrate that our coordinate descent method is scalable and provides high-quality solutions.  
 
 We showed in Theorem~\ref{thm:convergence} that our coordinate descent algorithm converges. It would be of interest to characterize the speed of convergence. In addition, the computational complexity of our algorithm may be improved by updating blocks of variables instead of one coordinate at a time. Further, an open question is whether, in the context of our guarantees in Theorem~\ref{thm:obj}, the sample size requirement and the rates of convergence to an optimal solution can be improved. Finally, another open question is to provide statistical consistency guarantees for our algorithm. \cite{xu2024integer} showed that, for sufficiently large $n$, any feasible $\Gamma$ for \eqref{Problem:micp} whose objective value differs from the optimum by at most $\mathcal{O}(\lambda^2)$ also belongs to the population Markov equivalence class. According to Theorem~\ref{thm:obj}, the gap achieved by coordinate descent can be much larger. It would be interesting to investigate whether under some orderings $O$, the gap obtained by coordinate descent can be on the order $\lambda^2$ so that we can attain consistency. Indeed, the synthetic result in Figure~\ref{fig:obj_diff} (left) suggests that some orderings have much faster convergence to optimality than others.
 
\clearpage
\appendix

\section{Proof of Theorem~\ref{thm:convergence}}
\label{proof:convergence}
The proof of Theorem~\ref{thm:convergence} relies on the following definitions and lemmas, and it closely follows the approach outlined in \cite{Rahul20coordinate}. With a slight abuse of notation, we let $\ell(\Gamma) := \sum_{i=1}^m-2\log(\Gamma_{ii})+\mathrm{tr}(\Gamma\Gamma^\T\hat{\Sigma}_n)$ to be the negative log-likelihood function associated with parameter $\Gamma \in \mathbb{R}^{m \times m}$.
\begin{definition}(Coordinate-wise (CW) minimum \citep{Rahul20coordinate})
A connectivity matrix $\Gamma^\mathrm{CW} \in \R^{m\times m}$ of a DAG is the CW minimum of problem \eqref{Problem:micp} if for every $(u,v), u,v=1,\ldots,m$, $\Gamma^\mathrm{CW}_{uv}$
is a minimizer of  $g(\Gamma_{uv})$ with other coordinates of $\Gamma^\mathrm{CW}$ held fixed.
\end{definition}
\begin{lemma}
\label{lem:decrease}
    Let $\{\Gamma^{j}\}_{j=1}^\infty$ be the sequence generated by Algorithm \ref{algo:cd_spacer}. Then the sequence of objective values $\{f(\Gamma^{j})\}_{j=1}^\infty$ is decreasing and converges.
\end{lemma}
\begin{proof}
The Hessian matrix of function $\ell(\Gamma)$ with respect to $\Gamma$ is $2\mathrm{diag}\left(1/\Gamma_{11}^2, \ldots, 1/\Gamma_{mm}^2\right) + 2\hat{\Sigma}$. Therefore, by Assumption \ref{assumption:convexity}, $\ell(\Gamma)$ is strongly convex and thus bounded below, and so is $f(\Gamma)$.
If $\Gamma^j$ is the result of a non-spacer step, then the inequality $f(\Gamma^j) \leq f(\Gamma^{j-1})$ holds trivially. Similarly, we know that if $\Gamma^j$ results from a spacer step, then, $\ell(\Gamma^j) \leq \ell(\Gamma^{j-1})$. Since a spacer step updates only coordinates on the support, it cannot increase the support size of $\Gamma^{j-1}$, that is, $\norm{\Gamma^j-\mathrm{diag}(\Gamma^j)}_{\ell_0} \leq \norm{\Gamma^{j-1}-\mathrm{diag}(\Gamma^{j-1})}_{\ell_0}$, thus $f(\Gamma^j) \leq f(\Gamma^{j-1})$. Since $f(\Gamma^j)$ is non-increasing and bounded below, it must converge.
\end{proof}

\begin{lemma}
\label{lem:bounded}
     The sequence $\{\Gamma^t\}_{t=1}^\infty$ generated by Algorithm \ref{algo:cd_spacer} is bounded.
\end{lemma}
\begin{proof}
    By Algorithm \ref{algo:cd_spacer}, $\Gamma^t \in G := \{\Gamma \in \mathbb{R}^{m\times m} \mid f(\Gamma) \leq f(\Gamma^0)\}$. It suffices to show that the set $G$ is bounded. From Proposition  11.11 in \cite{Bauschke2011ConvexAA}, if the function $f$ is coercive, then the set $G$ is bounded. Since $f(\Gamma) \geq \ell(\Gamma)$ for every $\Gamma$, it suffices to show that the function $\ell$ is coercive. By Assumption \ref{assumption:convexity}, we have that the function $\ell$ is strongly convex. The lemma then follows from the classical result in convex analysis that strongly convex functions are coercive.    

\end{proof}

The following lemma characterizes the limit points of Algorithm \ref{algo:cd_spacer}.

\begin{lemma}
\label{lem:converge}
    Let $\hat{E}$ be a support set that is generated infinitely often by the non-spacer steps of Algorithm~\ref{algo:cd_spacer}, and let $\left\{\Gamma^l\right\}_{l \in L}$ be the estimates from the spacer steps when the support of the input matrix is $\hat{E}$. Then:
    \begin{enumerate}
        \item There exists a positive integer $M$ such that for all $l \in L$ with $l \geq M$, $\mathrm{support}(\Gamma^l)=\hat{E}$.
        \item There exists a subsequence of $\left\{\Gamma^l\right\}_{l \in L}$ that converges to a stationary solution ${\Gamma}^\mathrm{CW}$, where, ${\Gamma}^\mathrm{CW}$ is the unique minimizer of $\min _{\mathrm{support}(\Gamma) \subseteq \hat{E}} \ell(\Gamma)$.
        \item Every subsequence of $\{\Gamma^t\}_{t \geq 0}$ with support $\hat{E}$ converges to $\Gamma^\mathrm{CW}$.
    \end{enumerate}
\end{lemma}
\begin{proof}
\textbf{Part 1.)} Since spacer steps optimize only over the coordinates in $\hat{E}$, no element outside $\hat{E}$ can be added to the support. Thus, for every $l \in L$ we have $\mathrm{support}(\Gamma^l) \subseteq \hat{E}$. We next show that strict containment is not possible via contradiction. Suppose $ \mathrm{support}(\Gamma^l) \subsetneq \hat{E}$ occurs infinitely often, and consider some $l \in L$ where this occurs. By the spacer step of Algorithm \ref{algo:cd_spacer}, the previous iterate $\Gamma^{l-1}$ has support $\hat{E}$, implying $\left\|\Gamma^{l-1}\right\|_0-\left\|\Gamma^l\right\|_0 \geq$ 1. Moreover, from the definition of the spacer step, we have $\ell(\Gamma^l) \leq \ell(\Gamma^{l-1})$. Therefore, we get
$f(\Gamma^{l-1})-f(\Gamma^l)=\ell(\Gamma^{l-1})-\ell(\Gamma^l)+\lambda^2(\|\Gamma^{l-1}\|_0-\|\Gamma^l\|_0) \geq \lambda^2$.
Thus, when $\mathrm{support}(\Gamma^l) \subsetneq \hat{E}$ occurs, $f$ decreases by at least $\lambda^2$. Therefore, $\Gamma^l \subsetneq \hat{E}$ infinitely many times implies that $f(\Gamma)$ is not lower-bounded, which is a contradiction. 

\textbf{Part 2.)} The proof follows the conventional procedure for establishing the convergence of cyclic coordinate descent (CD) \citep{bertsekas2016nonlinear,Rahul20coordinate}. 
We obtain $\Gamma^{l}$ by updating every coordinate in $\hat{E}$ of $\Gamma^{l-1}$. Denote the intermediate steps as $\Gamma^{l,1},\ldots,\Gamma^{l,|\hat{E}|}$, where $\Gamma^{l,|\hat{E}|} = \Gamma^l$. We aim to show that the sequence $\{\Gamma^{l,|\hat{E}|}\}_{l\in L}$ converges to a point $\Gamma^\mathrm{CW}$, and similarly, other sequences $\{\Gamma^{l,i}\}_{l\in L}, i=1,\ldots, |\hat{E}|-1$, also converge to $\Gamma^\mathrm{CW}$.  By Lemma \ref{lem:bounded}, since $\{\Gamma^{l,|\hat{E}|}\}_{l\in L}$ is a bounded sequence, there exists a converging subsequence $\{\Gamma^{l',|\hat{E}|}\}_{l'\in L'}$ with a limit point $\Gamma^\mathrm{CW}$. Without loss of generality, we choose the subsequence satisfying $l' > M$, $\forall l'\in L'$. From Part 1 of the lemma, $\{\Gamma^{l', 1}\}_{l'\in L'}, \ldots, \{\Gamma^{l',|\hat{E}|-1}\}_{l'\in L'}$ all have the same support $\hat{E}$. For $\{\Gamma^{l', |\hat{E}|-1}\}_{l'\in L'}$, we have $f(\Gamma^{l',|\hat{E}|-1})-f(\Gamma^{l',|\hat{E}|})=\ell(\Gamma^{l'|\hat{E}|-1})-\ell(\Gamma^{l',|\hat{E}|})$. If the change from $\Gamma^{l',|\hat{E}|-1}$ to $\Gamma^{l', |\hat{E}|}$ is on a diagonal entry, say $\Gamma_{uu}$, then,  

\begin{equation}
\begin{aligned}
    &\ell\left(\Gamma^{l',|\hat{E}|-1}\right)-\ell\left(\Gamma^{l',|\hat{E}|}\right) \\
    =& -2\log{(\Gamma_{uu}^{l',|\hat{E}|-1}}/{\Gamma_{uu}^{l',|\hat{E}|}})+\left(\left(\Gamma_{uu}^{l',|\hat{E}|-1}\right)^2 - \left(\Gamma_{uu}^{l',|\hat{E}|}\right)^2\right)\hat{\Sigma}_{uu}+(\Gamma_{uu}^{l',|\hat{E}|-1}-\Gamma_{uu}^{l',|\hat{E}|})A_{uu}\\
    =&-2\log{(\Gamma_{uu}^{l',|\hat{E}|-1}}/{\Gamma_{uu}^{l',|\hat{E}|}})+\left(\Gamma_{uu}^{l',|\hat{E}|-1} - \Gamma_{uu}^{l',|\hat{E}|}\right)^2\hat{\Sigma}_{uu} - 2\left(\Gamma_{uu}^{l',|\hat{E}|}\right)^2\hat{\Sigma}_{uu} \\
    &+2\Gamma_{uu}^{l',|\hat{E}|-1}\Gamma_{uu}^{l',|\hat{E}|}\hat{\Sigma}_{uu}+(\Gamma_{uu}^{l',|\hat{E}|-1}-\Gamma_{uu}^{l',|\hat{E}|})A_{uu}\\
    =&-2\log{(\Gamma_{uu}^{l',|\hat{E}|-1}}/{\Gamma_{uu}^{l',|\hat{E}|}})+\left(\Gamma_{uu}^{l',|\hat{E}|-1} - \Gamma_{uu}^{l',|\hat{E}|}\right)^2\hat{\Sigma}_{uu}\\
    &+2(\Gamma_{uu}^{l',|\hat{E}|-1}-\Gamma_{uu}^{l',|\hat{E}|})\Gamma_{uu}^{l',|\hat{E}|}\hat{\Sigma}_{uu}+(\Gamma_{uu}^{l',|\hat{E}|-1}-\Gamma_{uu}^{l',|\hat{E}|})A_{uu}.
\end{aligned}
\label{eq:diff_uu}
\end{equation}

From the update rule in Proposition~\ref{prop:update}, we have 
\begin{equation*}
    (4\Gamma_{uu}^{l',|\hat{E}|}\hat{\Sigma}_{uu} + A_{uu})^2 = A_{uu}^2 + 16\hat{\Sigma}_{uu},
\end{equation*}
which gives
\begin{equation}
    2\Gamma_{uu}^{l',|\hat{E}|}\hat{\Sigma}_{uu} + A_{uu} = \frac{2}{\Gamma_{uu}^{l',|\hat{E}|}}.
    \label{eq:update_uu_rearrange}
\end{equation}

Combining \eqref{eq:diff_uu} and \eqref{eq:update_uu_rearrange}, we have
\begin{equation}
\begin{aligned}
&\ell\left(\Gamma^{l',|\hat{E}|-1}\right)-\ell\left(\Gamma^{l',|\hat{E}|}\right) \\
 =&2\left(-\log{(\Gamma_{uu}^{l',|\hat{E}|-1}}/{\Gamma_{uu}^{l',|\hat{E}|})} + {\Gamma_{uu}^{l',|\hat{E}|-1}}/{\Gamma_{uu}^{l',|\hat{E}|}} - 1\right) + \left(\Gamma_{uu}^{l',|\hat{E}|-1} - \Gamma_{uu}^{l',|\hat{E}|}\right)^2\hat{\Sigma}_{uu}.
\end{aligned}
\label{eq:diff_uu_positive}
\end{equation}
Since $a-1 \geq \log(a)$ for $a\geq 0$, each of the two terms above is non-negative. From Lemma~\ref{lem:decrease}, as $l' \to \infty$, $f(\Gamma^{l',|\hat{E}|-1})-f(\Gamma^{l',|\hat{E}|})$ or equivalently $\ell(\Gamma^{l',|\hat{E}|-1})-\ell(\Gamma^{l',|\hat{E}|})$ converges to $0$ as $l' \rightarrow \infty.$
Combining this with the fact that $\ell(\Gamma^{l',|\hat{E}|-1})-\ell(\Gamma^{l',|\hat{E}|}) \geq 0$ and that each term in \eqref{eq:diff_uu_positive} is non-negative, we conclude that $\Gamma^{l',|\hat{E}|-1}$ must converge to $\Gamma^{l',|\hat{E}|}$ as $l' \to \infty$. Since $\Gamma^{l',|\hat{E}|}$ converges to $\Gamma^\mathrm{CW}$, $\Gamma^{l',|\hat{E}|-1}$ must also converge to $\Gamma^\mathrm{CW}$. Repeating a similar argument, we conclude that $\Gamma^{l',j}$ converges to $\Gamma^\mathrm{CW}$ for all $j= 1,2,\dots,|\hat{E}|$.

If the change from $\Gamma^{l',|\hat{E}|-1}$ to $\Gamma^{l', |\hat{E}|}$ is on an off-diagonal entry, say $\Gamma_{uv}$ with $u\neq v$, then, after some algebra, 
\begin{equation*}
f\left(\Gamma^{l',|\hat{E}|-1}\right)-f\left(\Gamma^{l',|\hat{E}|}\right)= \ell\left(\Gamma^{l',|\hat{E}|-1}\right)-\ell\left(\Gamma^{l',|\hat{E}|}\right) = \left(\Gamma_{uv}^{l',|\hat{E}|-1} - \Gamma_{uv}^{l',|\hat{E}|}\right)^2\hat{\Sigma}_{uu}\,.
\end{equation*}
Again, appealing to Lemma~\ref{lem:decrease} as before, we can conclude that $\Gamma^{l',|\hat{E}|-1}$ converges to $\Gamma^\mathrm{CW}$ as $l' \to \infty$. Similarly, $\Gamma^{l',j}$ converges to $\Gamma^\mathrm{CW}$ for every $j = 1,2,\dots,|\hat{E}|-1$.

Consider $k, l \in L'$ with $k > l$ such that for the $j$-th coordinate in $\hat{E}$, $f(\Gamma^{k})\leq f(\Gamma^{l, j})\leq f(\tilde{\Gamma}^{l, j})$. Here, $\tilde{\Gamma}^{l, j}$ equals to 
$\Gamma^{l, j}$ except for the $j$-th nonzero coordinate in $\hat{E}$. As $k, l \rightarrow \infty$, we have, from the above analysis, that there exists a matrix $\Gamma^\mathrm{CW}$ such that 
$\Gamma^{k} \to \Gamma^\mathrm{CW}$ and $\Gamma^{l,j} \to \Gamma^\mathrm{CW}$. Thus, $\Gamma^\mathrm{CW}$ and $\lim_{l \to \infty}\tilde{\Gamma}^{l, j}$ differ by only one coordinate in the $j$-th position. We conclude that $f(\Gamma^\mathrm{CW}) \leq f(\lim_{l \to \infty}\tilde{\Gamma}^{l, j})$. In other words, $\Gamma^\mathrm{CW}$ is coordinate-wise minimum. Furthermore, since the optimization problem $\min _{\mathrm{support}(\Gamma) \subseteq \hat{E}} \ell(\Gamma)$ is strongly convex by Assumption~\ref{assumption:convexity}, $\Gamma^\mathrm{CW}$ is the unique minimizer of this optimization problem. 

\textbf{Part 3.)} Consider any subsequence $\{\Gamma^k\}_{k\in K}$ such that $\mathrm{support}(\Gamma^k)=\hat{E}$. We will show by contradiction that $\{\Gamma^k\}_{k\in K}$ must converge to $\Gamma^\mathrm{CW}$. Suppose $\{\Gamma^k\}_{k\in K}$ has a limit point $\hat{\Gamma} \not = \Gamma^\mathrm{CW}$. Then there exist a subsequence $\{\Gamma^{k'}\}_{k'\in K'}$, with $K' \subseteq K$, that converges to $\hat{\Gamma}$. Therefore,  $\lim_{k'\rightarrow \infty} f(\Gamma^{k'}) = \ell(\hat{\Gamma}) + \lambda^2 |\hat{E}|.$
From part 1 and part 2, we have that, for the subsequence $\{\Gamma^{l'}\}$, $\lim_{l'\rightarrow \infty} f(\Gamma^{l'}) = \ell(\Gamma^\mathrm{CW}) + \lambda^2 |\hat{E}|.$ 
By Lemma \ref{lem:decrease}, we have $\lim_{k'\rightarrow \infty} f(\Gamma^{k'}) = \lim_{l'\rightarrow \infty} f(\Gamma^{l'})$. Thus, we conclude that $\ell(\hat{\Gamma}) = \ell(\Gamma^\mathrm{CW})$, which contradicts the fact that $\Gamma^\mathrm{CW}$ is the unique minimizer of $\min_{\mathrm{support}(\Gamma)\subseteq \hat{E}}\ell(\Gamma)$. Therefore, we conclude that any subsequence with support $\hat{E}$ converges to $\Gamma^\mathrm{CW}$ as $k\rightarrow \infty$.
\end{proof}

The next lemma shows that if the limit point of $\{\Gamma^k\}$ has a certain support, then that support must appear infinitely often in the sequence. 
\begin{lemma}
\label{lem:io}
    Let $\Gamma$ be a limit point of $\{\Gamma^k\}_{k=1}^\infty$ with $\mathrm{support}(\Gamma)=E$. We have $\mathrm{support}(\Gamma^k)=E$ for infinitely many $k$'s.
\end{lemma}
\begin{proof}
    We prove this result by contradiction. Assume that there are only finitely many $k$'s such that $\mathrm{support}(\Gamma^k)=E$. 
    Since there are finitely many possible support sets, there is a support $E'\not=E$ and a subsequence $\{\Gamma^{k'}\}$ of $\{\Gamma^k\}$ such that $\mathrm{support}(\Gamma^{k'}) =E'$ for all $k'$, and $\lim_{k'\rightarrow \infty} \Gamma^{k'} = \Gamma$. However, by Lemma \ref{lem:converge}, the subsequence converges to a minimizer $\Gamma^\mathrm{CW}$ with $\mathrm{support}(\Gamma^\mathrm{CW}) = E'$ and thus $\Gamma^\mathrm{CW}\not=\Gamma$. This is a contradiction.
\end{proof}
This lemma rules out the possibility that a coordinate remains nonzero throughout the sequence, gradually decreasing to zero in the limit.

We are now ready to complete the proof of Theorem \ref{thm:convergence}.
\begin{proof}[Proof of Theorem \ref{thm:convergence}]
Let $\Gamma$ be a limit point of $\{\Gamma^k\}$ with  the largest support size and denote its support by
$\hat{E}$. By Lemma \ref{lem:io}, there is a subsequence $\{\Gamma^r\}_{r\in R}$ of $\{\Gamma^k\}$ such that $\mathrm{support}(\Gamma^r) = \hat{E}, \forall r \in R$, and $\lim_{r\rightarrow\infty} \Gamma^r = \Gamma$. By Lemma \ref{lem:converge}, there exists an integer $M$ such that for every $r \geq M $ and $r+1$ is a spacer step, we have $\mathrm{support}(\Gamma^r) = \mathrm{support}(\Gamma^{r+1})$. Without loss of generality, we choose the subsequence for which $r>M, \forall r \in R$. We will demonstrate by contradiction that any coordinate $(u,v)$ in $\hat{E}$ cannot be dropped infinitely often in $\{\Gamma^k\}$. 
To this end, assume that $(u,v)\not \in \{\mathrm{support}(\Gamma^{k})\}_{k > M}$ infinitely often. Let $\{\Gamma^{r'}\}_{r' \in R'}$, where $R' \subseteq R$, be the subsequence with $\mathrm{support}(\Gamma^{r'+1}) = \hat{E} \setminus \{(u,v)\}, \forall r'\in R'$. Since $r' > M$ and the support has been changed, $r'+1$ is not a spacer step. Therefore, using Proposition \ref{prop:update}, we have $f(\Gamma^{r'}) - f(\Gamma^{r'+1}) =\ell(\Gamma^{r'})-\ell(\Gamma^{r'+1}) + \lambda^2 = \lambda^2 + \Gamma^{r'}_{uv}(\Gamma^{r'}_{uv}\hat{\Sigma}_{uv}+A_{uv}) = \lambda^2 - {A^2_{uv}}/{4\hat{\Sigma}_{uu}} > 0$, where we use the update rule that $\Gamma^{r'}_{uv} = -A_{uv}/(2\hat{\Sigma}_{uu})$. By Lemma \ref{lem:decrease}, we have $\lim_{r' \rightarrow \infty} f(\Gamma^{r'}) - f(\Gamma^{r'+1}) = 0$. Thus, $\lambda^2 = {A^2_{uv}}/{4\hat{\Sigma}_{uu}}$, where $A_{uv} = \sum_{j\not=u}\Gamma^{r'}_{jv}\hat{\Sigma}_{ju} + \sum_{k\not=u}\Gamma_{kv}^{r'}\hat{\Sigma}_{uk}$. By Proposition \ref{prop:update}, in step $r'+1$, we have $|\Gamma^{r'+1}_{uv}| = {\lambda}/{\sqrt{\hat{\Sigma}_{uu}}} > 0$, which contradicts the definition of $\{\Gamma^{r'}\}_{r' \in R'}$. Therefore, no coordinate in $\hat{E}$ can be dropped infinitely often. Moreover, no coordinate can be added to $\hat{E}$ infinitely often as $\hat{E}$ is the largest support. As a result, the support converges to $\hat{E}$. With stabilized support $\hat{E}$, by Lemma \ref{lem:converge}, we have that $\{\Gamma^k\}$ converges to the limit $\Gamma^\mathrm{CW}$ with support $\hat{E}$. From Algorithm \ref{algo:cd_spacer} and Proposition \ref{prop:update}, we have $\Gamma_{uv}$ is a minimizer of $f(\Gamma_{uv})$ with respect to the coordinate $(u,v)$ and others fixed. Therefore, $\Gamma^\mathrm{CW}$ is the CW minimum.

\end{proof}
\section{Proof of Theorem~\ref{thm:ordering}}
\label{proof:ordering}
To facilitate the proof, we first introduce several technical lemmas that extend the homoscedastic case of \cite{Chen19}'s Theorem~2 to our setting with mild heteroscedasticity. These lemmas rely on standard terminology and structural properties of directed acyclic graphs, which we recall below.

In a directed graph $\mathcal{G}(B^\star)$ induced by $B^\star$, if there is an edge $k \to j$, then $k$ is a parent of $j$, and $j$ is a child of $k$. We write $\mathrm{pa}(j)$ for the set of parents of $j$ and $\mathrm{ch}(j)$ for its children. If there is a directed path $k \to \cdots \to j$, then $k$ is an ancestor of $j$, and $j$ is a descendant of $k$. The sets of ancestors and descendants of $j$ are denoted by $\mathrm{an}(j)$ and $\mathrm{de}(j)$, respectively, with $j \in \mathrm{an}(j)$ and $j \in \mathrm{de}(j)$ by convention. A set $C$ is \emph{ancestral} if $\mathrm{an}(j) \subseteq C$ for all $j \in C$. 

The first lemma clarifies that the sources in $\mathcal{G}(B^\star)$ are characterized by minimal variances. 

\begin{lemma}\label{lemma:var_bound}
Let $X$ be the variables defined in \eqref{eqn:sem}, we have
\begin{itemize}
    \item If \( \mathrm{pa}(j) = \emptyset \), then \( \mathrm{var}(X_j) = \Omega_{jj}^\star \).
    \item If \( \mathrm{pa}(j) \neq \emptyset \), then \( \mathrm{var}(X_j) \geq \Omega_{jj}^\star + \Omega_{ll}^\star\zeta \) for some $l\in V$.
\end{itemize}
\end{lemma}

\begin{proof}
Recall that $X = (I - {B^\star}^\top)^{-1}\epsilon$. Let \( \Pi = (\pi_{jk}) = (I - {B^\star}^\top)^{-1} \), where $\pi_{jk}$ denotes the total causal effect from node 
$k$ to node $j$ in the DAG defined by $B^\star$. we have
\[
\mathrm{var}(X_j) =  \sum_{k=1}^m \Omega_{kk}^\star\pi_{jk}^2.
\]

If \( \mathrm{pa}(j) = \emptyset \), then \( \pi_{jk}^2 = 0 \) for all \( k \neq j \), and thus  
\[
\mathrm{var}(X_j) = \Omega^\star_{jj} \pi_{jj}^2 = \Omega^\star_{jj}.
\]

If \( \mathrm{pa}(j) \neq \emptyset \), then by acyclicity of \( \mathcal{G}(B^\star) \), there exists a node \( l \in \mathrm{pa}(j) \) such that \( \mathrm{de}(l) \cap \mathrm{pa}(j) = \{l\} \). Then  
\[
\pi_{jl}^2 = {B^\star_{jl}}^2 \geq \zeta,
\]
and  
\[
\mathrm{var}(X_j) = \Omega^\star_{jj} + \sum_{k \neq j}\Omega_{kk}^\star \pi_{jk}^2 \geq \Omega^\star_{jj} + \Omega_{ll}^\star\pi_{jl}^2 \geq \Omega^\star_{jj} + \Omega_{ll}^\star\zeta.
\]
\end{proof}

The next lemma shows that by conditioning on a source, or more generally, an ancestral set, one recovers a structural equation model whose associated graph has the source node or the entire ancestral set removed. For a variable \( X_j \) and a vector \( X_C = (X_k : k \in C) \), we define \( X_{j,C} = X_j - \mathbb{E}(X_j|X_C) \).

\begin{lemma}
Let \( C \) be an ancestral set in \(\mathcal{G}(B^\star)\). Then \((X_{j,C} : j \notin C) \) is associated with the submatrix \( B^\star[-C] = (\beta_{jk})_{j,k \notin C} \).
\end{lemma}

\begin{proof}
Let \( j \notin C \). Since \( C \) is ancestral, \( X_C \) is a function of \( \epsilon_C = (\epsilon_k:k\in C) \) only and thus independent of \( \epsilon_j \). Hence, \( \mathbb{E}(\epsilon_j|X_C) = \mathbb{E}(\epsilon_j) = 0 \). Because it also holds that \( X_{k,C} = 0 \) for \( k \in C \), we have that

\[
X_{j,C} = \sum_{k \in \text{pa}(j) \setminus C} \beta_{jk} X_{k,C} + \epsilon_j.
\]
\end{proof}






Before proving Theorem~\ref{thm:ordering}, we give a lemma that addresses the estimation error for inverse covariances.

\begin{lemma}[Lemma~6 in \cite{Chen19}]\label{lemma:estimation_error}
Suppose all $(q+1) \times (q+1)$ principal submatrices of the covariance matrix \(\Sigma^\star = \mathbb{E}(XX^T)\) has minimum eigenvalue at least \(\lambda_{\min} > 0\). If for \(\epsilon, \eta > 0\) we have  
\begin{equation}
\label{eqn:n_require}
n > (q+1)^2 \left\{ \log \left( \frac{2m^2 + 2m}{\epsilon} \right) \right\} 128 \left( 1 + 4 \frac{\Omega^\star_{\max}}{\Omega^\star_{\min}} \right)^2 \left( \max_{j \in V} \Sigma^\star_{jj} \right)^2 \left( \frac{\eta \lambda_{\min} + 1}{\eta \lambda_{\min}^2} \right)^2,
\end{equation}
then  
\[
\max_{C \subseteq V, |C|\leq q+1} \| (\Sigma^\star_{C,C})^{-1} - (\hat{\Sigma}_{C,C})^{-1} \|_{\infty} \leq \eta
\]  
with probability at least \(1 - \epsilon\).
\end{lemma}

\ignore{
\begin{proof}
Let \(\delta = \frac{\eta \lambda_{\min}^2}{(q+1)(\eta \lambda_{\min} + 1)}\). Because \(\delta < \frac{\lambda_{\min}}{q+1}\), by Lemma~5 from \cite{harris2013pc}, we have  
\[
\max_{C \subseteq V, |C| \leq (q+1)} \| (\Sigma^\star_{C,C})^{-1} - (\hat{\Sigma}_{C,C})^{-1} \|_{\infty} \leq \frac{(q+1)\delta / \lambda_{\min}^2}{1 - (q+1)\delta / \lambda_{\min}} = \eta,
\]  
provided \(\|\hat{\Sigma} - \Sigma^\star\|_{\infty} \leq \delta\). The proof is thus complete if we show that \(\mathbb{P} \left( \|\hat{\Sigma} - \Sigma^\star\|_{\infty} > \delta \right) \leq \epsilon\).

Note that \(X_{j}=\varepsilon_{j}+\sum_{k\in\operatorname{an}(j)}\pi_{jk}\varepsilon_{k}\) has variance \(\Omega_{jj}^\star+\sum_{k\in\operatorname{an}(j)}\pi^2_{jk}\Omega_{kk}^\star\). Since \(\Omega_{\max}^\star\) is a bound on the variances of all \(\epsilon_{j}\), it follows that \(X_{j}/\sqrt{\operatorname{var}(X_{j})}\) is sub-Gaussian with parameter at most \(\Omega^\star_{\max}/\Omega^\star_{\min}\). Lemma~1 of \cite{Ravikumar2008HighdimensionalCE} applies and gives  
\[
\mathbb{P}\{\left|\widehat{\Sigma}_{i,j}-\Sigma^\star_{i,j}\right|>\delta\}\leq 4 \exp\left\{-\frac{n\delta^{2}}{128(1+4\Omega^\star_{\max}/\Omega^\star_{\min})^{2}\max_{j}(\Sigma^\star_{j,j})^{2}}\right\}\leq\frac{2}{m(m+1)}\epsilon.
\]  
A union bound over the entries of \(\Sigma\) yields that indeed \(\mathbb{P}\left(\left\|\widehat{\Sigma}-\Sigma^\star\right\|_{\infty}>\delta\right)\leq \epsilon\).
\end{proof}
}

\begin{proof}[Proof of Theorem~\ref{thm:ordering}]
Our assumption on \( n \) is as in \eqref{eqn:n_require} with \( \eta = \frac{1}{2\Omega^\star_{\max}} - \frac{1}{2\Omega^\star_{\min}(1 + \zeta)} \). Lemma~\ref{lemma:estimation_error} thus implies that, with probability at least \( 1 - \epsilon \), we have, for all subsets \( C \subseteq V \),  
\begin{equation}\label{eq:bound}
\left\|(\hat{\Sigma}_{C,C})^{-1} - (\Sigma_{C,C})^{-1}\right\|_\infty \leq \frac{1}{2\Omega^\star_{\max}} - \frac{1}{2\Omega^\star_{\min}(1 + \zeta)}.
\end{equation}

Let \( j \) be a source in \( \mathcal{G}(B^\star) \), and let \( k \) be a non-source. Note that the variance of \( j \) conditional on some set \( C \) is
\[
\Omega^\star_{j|C} = \frac{1}{\left\{ (\Sigma_{C \cup \{j\}, C \cup \{j\}})^{-1} \right\}_{j,j}}.
\]

For an ancestral set \( C \subseteq V \setminus \{j,k\} \) such that \( \operatorname{pa}(j) \subseteq C \) and \( \operatorname{pa}(k) \not\subseteq C \),
\begin{equation}\label{eq:precision_diff}
\left\{ (\Sigma_{C \cup \{j\}, C \cup \{j\}})^{-1} \right\}_{j,j} - \left\{ (\Sigma_{C \cup \{k\}, C \cup \{k\}})^{-1} \right\}_{k,k} \geq \frac{1}{\Omega^\star_{jj}} - \frac{1}{\Omega^\star_{kk} + \Omega^\star_{\min}\zeta} \geq  \frac{1}{\Omega^\star_{\max}} - \frac{1}{\Omega^\star_{\min}(1 + \zeta)}.
\end{equation}

Using \eqref{eq:bound}, we obtain that
\begin{equation}\label{eq:est_diff}
\left\{ (\hat{\Sigma}_{C \cup \{j\}, C \cup \{j\}})^{-1} \right\}_{j,j} - \left\{ (\hat{\Sigma}_{C \cup \{k\}, C \cup \{k\}})^{-1} \right\}_{k,k} > 0.
\end{equation}

Thus, \( \hat{\Omega}_{j|C}^2 < \hat{\Omega}_{k|C}^2 \), which implies that Algorithm~\ref{algo:TO} correctly selects a source node at each step. In the first step, \( C = \emptyset \), which is trivially an ancestral set. By induction, each subsequent step then correctly adds a sink to \( C \), so \( C \) remains ancestral and a correct ordering is recovered.
\end{proof}

\section{Results of near homoscedastic case}
\label{sec:extra_experiment1}
Table~\ref{tab:compare_benchmarks_est_small} presents results for a near-homoscedastic scenario, where the diagonal entries of $\Omega^\star$ are sampled uniformly from $\{0.8, 1, 1.2\}$. Because the variance differences are small, the TD method is able to identify an good topological ordering, leading to comparable performance across methods.

\begin{table}[tb]
\def~{\hphantom{0}}
    \centering
    \caption{Comparison of our method, CD-$\ell_0$, with competing methods}
    \resizebox{\columnwidth}{!}{\begin{tabular}{lccccccccccccc}

        & \multicolumn{3}{c}{$\mathrm{MICODAG}$} & \multicolumn{2}{c}{CCDr-MCP} & \multicolumn{2}{c}{GES} & \multicolumn{2}{c}{TD} & \multicolumn{2}{c}{NOTEARS} & \multicolumn{2}{c}{CD-$\ell_0$} \\

        Network($m$) & Time & \rgap & $d_{\mathrm{cpdag}}$ & Time & $d_{\mathrm{cpdag}}$ & Time & $d_{\mathrm{cpdag}}$ & Time & $d_{\mathrm{cpdag}}$ & Time & $d_{\mathrm{cpdag}}$ & Time & $d_{\mathrm{cpdag}}$ \\

        Dsep(6)       & $\leq 1$   & 0     & 2.0$(\pm0)$    
                      & $\leq 1$   & 2.0$(\pm0)$    
                      & $\leq 1$   & 2.0$(\pm0)$    
                      & $\leq 1$   & 2.0$(\pm0)$    
                      & $\leq 1$   & 2.0$(\pm0)$    
                      & $\leq 1$   & 2.0$(\pm0)$ \\

        Asia(8)       & $\leq 1$   & 0     & 2.0$(\pm0)$    
                      & $\leq 1$   & 2.3$(\pm0.95)$    
                      & $\leq 1$   & 2.1$(\pm0.32)$    
                      & $\leq 1$   & 3.2$(\pm1.93)$    
                      & $\leq 1$   & 2.8$(\pm1.75)$    
                      & $\leq 1$   & 3.2$(\pm1.93)$ \\

        Bowling(9)    & $\leq 1$   & 0     & 2.0$(\pm0)$    
                      & $\leq 1$   & 3.1$(\pm1.91)$    
                      & $\leq 1$   & 2.4$(\pm1.26)$    
                      & $\leq 1$   & 2.0$(\pm0)$    
                      & $\leq 1$   & 3.8$(\pm0.63)$    
                      & $\leq 1$   & 2.0$(\pm0)$ \\

        InsSmall(15)  & 569       & .007  & 12.2$(\pm2.74)$   
                      & $\leq 1$   & 27.0$(\pm4.08)$   
                      & $\leq 1$   & 20.2$(\pm5.69)$   
                      & $\leq 1$   & 8.0$(\pm0)$    
                      & $\leq 1$   & 6.7$(\pm1.49)$   
                      & $\leq 1$   & 8.0$(\pm0)$ \\

        Rain(14)      & 36.7       & 0     & 2.2$(\pm0.63)$    
                      & $\leq 1$   & 8.5$(\pm2.12)$    
                      & $\leq 1$   & 4.2$(\pm3.33)$   
                      & $\leq 1$   & 2.0$(\pm0)$    
                      & $\leq 1$   & 3.4$(\pm2.50)$    
                      & $\leq 1$   & 2.0$(\pm0)$ \\

        Cloud(16)     & 30.0       & 0     & 5.0$(\pm0)$    
                      & $\leq 1$   & 8.5$(\pm4.45)$    
                      & $\leq 1$   & 5.0$(\pm0)$   
                      & $\leq 1$   & 3.2$(\pm1.03)$   
                      & $\leq 1$   & 19.1$(\pm2.73)$   
                      & $\leq 1$   & 5.2$(\pm2.62)$ \\

        Funnel(18)    & 14.4       & 0     & 2.1$(\pm0.32)$    
                      & $\leq 1$   & 3.3$(\pm1.70)$    
                      & $\leq 1$   & 8.9$(\pm10.29)$    
                      & $\leq 1$   & 2.0$(\pm0)$    
                      & $\leq 1$   & 3.2$(\pm1.55)$    
                      & $\leq 1$   & 2.0$(\pm0)$ \\

        Galaxy(20)    & 125.4      & 0     & 1.0$(\pm0)$    
                      & $\leq 1$   & 9.1$(\pm4.33)$    
                      & $\leq 1$   & 2.1$(\pm3.48)$    
                      & $\leq 1$   & 1.2$(\pm0.63)$    
                      & $\leq 1$   & 17.0$(\pm2.91)$    
                      & $\leq 1$   & 1.2$(\pm0.63)$ \\

        Insurance(27) & 1379.5     & .220  & 15.8$(\pm5.12)$  
                      & $\leq 1$   & 32.5$(\pm6.04)$   
                      & $\leq 1$   & 28.8$(\pm4.61)$   
                      & 0.124       & 11.7$(\pm2.67)$   
                      & $\leq 1$   & 12.7$(\pm5.93)$   
                      & 0.111       & 10.3$(\pm1.49)$ \\

        Factors(27)   & 1490.9     & .217  & 65.1$(\pm6.69)$  
                      & $\leq 1$   & 65.1$(\pm6.37)$   
                      & $\leq 1$   & 67.0$(\pm8.26)$   
                      & $\leq 1$   & 18.7$(\pm3.92)$   
                      & $\leq 1$   & 18.4$(\pm4.27)$   
                      & $\leq 1$   & 20.9$(\pm4.28)$ \\

        Hfinder(56)   & 3940.4     & .156  & 12.9$(\pm4.09)$   
                      & $\leq 1$   & 14.9$(\pm3.84)$   
                      & $\leq 1$   & 29.6$(\pm11.65)$  
                      & 1.50       & 14.1$(\pm7.34)$   
                      & 1.66       & 18.2$(\pm3.88)$   
                      & $\leq 1$   & 14.7$(\pm9.19)$ \\

        Hepar(70)     & 4132.8     & .732  & 32.6$(\pm6.54)$     
                      & $\leq 1$   & 51.6$(\pm9.50)$   
                      & $\leq 1$   & 76.6$(\pm29.45)$  
                      & 3.60       & 18.5$(\pm7.37)$   
                      & 3.21       & 42.6$(\pm5.68)$   
                      & 1.23       & 11.6$(\pm6.33)$ \\
    \end{tabular}}
    {\raggedright \footnotesize {Here, MICODAG, mixed-integer convex program \citep{xu2024integer}; CCDr-MCP, minimax concave penalized estimator with coordinate descent \citep{aragam2019learning}; GES, greedy equivalence search algorithm \citep{chickering2002optimal}; TD, top-down method \citep{Chen19}; NOTEARS, \citep{zheng2018dags}; $d_{\mathrm{cpdag}}$, differences between the true and estimated completed partially directed acyclic graphs; \rgap, relative optimality gap. All results are computed over ten independent trials where the average $d_\mathrm{cpdag}$ values are presented with their standard deviations.}\par}
    \label{tab:compare_benchmarks_est_small}
\end{table}



\acks{We thank the AE and the two referees for their constructive comments that improved this paper. We also thank Rohit Bhattacharya for his thoughtful feedback on an earlier version of this manuscript, which led to improvements in the paper. Simge Küçükyavuz and Tong Xu are supported, in part, by the Office of Naval Research Global [Grant N00014-22-1-2602]. Armeen Taeb is supported by NSF [Grant DMS-2413074] and by the Royalty Research Fund at the University of Washington. Ali Shojaie is supported by grant R01GM133848 from the  National Institutes of Health.}

\vskip 0.2in
\bibliography{reference}

\begin{thebibliography}{30}
\providecommand{\natexlab}[1]{#1}
\providecommand{\url}[1]{\texttt{#1}}
\expandafter\ifx\csname urlstyle\endcsname\relax
  \providecommand{\doi}[1]{doi: #1}\else
  \providecommand{\doi}{doi: \begingroup \urlstyle{rm}\Url}\fi

\bibitem[Aragam and Zhou(2015)]{aragam2015concave}
Bryon Aragam and Qing Zhou.
\newblock Concave penalized estimation of sparse {Gaussian} {Bayesian} networks.
\newblock \emph{Journal of Machine Learning Research}, 16\penalty0 (1):\penalty0 2273--2328, 2015.

\bibitem[Aragam et~al.(2019)Aragam, Gu, and Zhou]{aragam2019learning}
Bryon Aragam, Jiaying Gu, and Qing Zhou.
\newblock Learning large-scale {Bayesian} networks with the sparsebn package.
\newblock \emph{Journal of Statistical Software}, 91:\penalty0 1--38, 2019.

\bibitem[Bauschke and Combettes(2011)]{Bauschke2011ConvexAA}
Heinz~H. Bauschke and Patrick~L. Combettes.
\newblock Convex analysis and monotone operator theory in {Hilbert} spaces.
\newblock In \emph{CMS Books in Mathematics}, 2011.

\bibitem[Behdin et~al.(2023)Behdin, Chen, and Mazumder]{behdin2023sparse}
Kayhan Behdin, Wenyu Chen, and Rahul Mazumder.
\newblock Sparse gaussian graphical models with discrete optimization: Computational and statistical perspectives.
\newblock \emph{arXiv preprint arXiv:2307.09366}, 2023.

\bibitem[Bertsekas(2016)]{bertsekas2016nonlinear}
Dimitri Bertsekas.
\newblock \emph{Nonlinear Programming}, volume~4.
\newblock Athena Scientific, 2016.

\bibitem[Chen et~al.(2019)Chen, Drton, and Wang]{Chen19}
Wenyu Chen, Mathias Drton, and Y~Samuel Wang.
\newblock {On causal discovery with an equal-variance assumption}.
\newblock \emph{Biometrika}, 106\penalty0 (4):\penalty0 973--980, 09 2019.
\newblock ISSN 0006-3444.
\newblock \doi{10.1093/biomet/asz049}.
\newblock URL \url{https://doi.org/10.1093/biomet/asz049}.

\bibitem[Chickering(2002)]{chickering2002optimal}
David~Maxwell Chickering.
\newblock Optimal structure identification with greedy search.
\newblock \emph{Journal of Machine Learning Research}, 3\penalty0 (Nov):\penalty0 507--554, 2002.

\bibitem[Ellis and Wong(2008)]{ellis2008learning}
Byron Ellis and Wing~Hung Wong.
\newblock Learning causal {Bayesian} network structures from experimental data.
\newblock \emph{Journal of the American Statistical Association}, 103\penalty0 (482):\penalty0 778--789, 2008.

\bibitem[Friedman et~al.(2007)Friedman, Hastie, H{\"o}fling, and Tibshirani]{Friedman07coordinate}
Jerome Friedman, Trevor Hastie, Holger H{\"o}fling, and Robert Tibshirani.
\newblock {Pathwise coordinate optimization}.
\newblock \emph{Annals of Applied Statistics}, 1\penalty0 (2):\penalty0 302 -- 332, 2007.

\bibitem[Friedman et~al.(2008)Friedman, Hastie, and Tibshirani]{friedman2008sparse}
Jerome Friedman, Trevor Hastie, and Robert Tibshirani.
\newblock Sparse inverse covariance estimation with the graphical lasso.
\newblock \emph{Biostatistics}, 9\penalty0 (3):\penalty0 432--441, 2008.

\bibitem[Fu and Zhou(2013)]{Fu13}
Fei Fu and Qing Zhou.
\newblock Learning sparse causal {Gaussian} networks with experimental intervention: Regularization and coordinate descent.
\newblock \emph{Journal of the American Statistical Association}, 108\penalty0 (501):\penalty0 288--300, 2013.

\bibitem[Gamella et~al.(2024)Gamella, Peters, and Bühlmann]{gamella2024causal}
Juan~L. Gamella, Jonas Peters, and Peter Bühlmann.
\newblock The causal chambers: Real physical systems as a testbed for {AI} methodology, 2024.
\newblock URL \url{https://arxiv.org/abs/2404.11341}.

\bibitem[Hazimeh and Mazumder(2020)]{Rahul20coordinate}
Hussein Hazimeh and Rahul Mazumder.
\newblock Fast best subset selection: Coordinate descent and local combinatorial optimization algorithms.
\newblock \emph{Operations Research}, 68\penalty0 (5):\penalty0 1517--1537, 2020.

\bibitem[Kalisch and B{{\"u}}hlmann(2007)]{kalisch2007estimating}
Markus Kalisch and Peter B{{\"u}}hlmann.
\newblock Estimating high-dimensional directed acyclic graphs with the {PC}-algorithm.
\newblock \emph{Journal of Machine Learning Research}, 8\penalty0 (22):\penalty0 613--636, 2007.

\bibitem[K\"{u}\c{c}\"{u}kyavuz et~al.(2023)K\"{u}\c{c}\"{u}kyavuz, Shojaie, Manzour, Wei, and Wu]{kucukyavuz2022consistent}
Simge K\"{u}\c{c}\"{u}kyavuz, Ali Shojaie, Hasan Manzour, Linchuan Wei, and Hao-Hsiang Wu.
\newblock Consistent second-order conic integer programming for learning {Bayesian} networks.
\newblock \emph{Journal of Machine Learning Research}, 24\penalty0 (322):\penalty0 1--38, 2023.

\bibitem[Lam and Fan(2009)]{lam2009sparsistency}
Clifford Lam and Jianqing Fan.
\newblock Sparsistency and rates of convergence in large covariance matrix estimation.
\newblock \emph{Annals of Statistics}, 37\penalty0 (6B):\penalty0 4254, 2009.

\bibitem[Manzour et~al.(2021)Manzour, K\"{u}\c{c}\"{u}kyavuz, Wu, and Shojaie]{Manzour21}
Hasan Manzour, Simge K\"{u}\c{c}\"{u}kyavuz, Hao-Hsiang Wu, and Ali Shojaie.
\newblock Integer programming for learning directed acyclic graphs from continuous data.
\newblock \emph{INFORMS Journal on Optimization}, 3\penalty0 (1):\penalty0 46--73, 2021.

\bibitem[Meinshausen and Buhlmann(2006)]{Meinshausen2006HighdimensionalGA}
Nicolai Meinshausen and Peter Buhlmann.
\newblock High-dimensional graphs and variable selection with the lasso.
\newblock \emph{Annals of Statistics}, 34:\penalty0 1436--1462, 2006.

\bibitem[Nandy et~al.(2018)Nandy, Hauser, and Maathuis]{nandy2018high}
Preetam Nandy, Alain Hauser, and Marloes Maathuis.
\newblock High-dimensional consistency in score-based and hybrid structure learning.
\newblock \emph{Annals of Statistics}, 46\penalty0 (6A):\penalty0 3151--3183, 2018.

\bibitem[Shah and Peters(2018)]{Shah2018TheHO}
Rajen~Dinesh Shah and J.~Peters.
\newblock The hardness of conditional independence testing and the generalised covariance measure.
\newblock \emph{Annals of Statistics}, 2018.

\bibitem[Silander and Myllym\"{a}ki(2006)]{silander2012simple}
Tomi Silander and Petri Myllym\"{a}ki.
\newblock A simple approach for finding the globally optimal {Bayesian} network structure.
\newblock In \emph{Proceedings of the Twenty-Second Conference on Uncertainty in Artificial Intelligence}, UAI'06, page 445–452, Arlington, Virginia, USA, 2006. AUAI Press.
\newblock ISBN 0974903922.

\bibitem[Spirtes et~al.(1993)Spirtes, Glymour, and Scheines]{causalitybase}
Peter Spirtes, Clark Glymour, and Richard Scheines.
\newblock \emph{Causation, Prediction, and Search}.
\newblock The MIT Press, 1993.
\newblock ISBN 978-1-4612-7650-0.

\bibitem[Tsamardinos et~al.(2006)Tsamardinos, Brown, and Aliferis]{Tsamardinos2006TheMH}
Ioannis Tsamardinos, Laura~E. Brown, and Constantin~F. Aliferis.
\newblock The max-min hill-climbing {Bayesian} network structure learning algorithm.
\newblock \emph{Machine Learning}, 65:\penalty0 31--78, 2006.

\bibitem[Uhler et~al.(2012)Uhler, Raskutti, Buhlmann, and Yu]{Uhler2012GeometryOT}
Caroline Uhler, Garvesh Raskutti, Peter Buhlmann, and Bin Yu.
\newblock Geometry of the faithfulness assumption in causal inference.
\newblock \emph{Annals of Statistics}, 41:\penalty0 436--463, 2012.

\bibitem[van~de Geer and B{\"u}hlmann(2013)]{vgBuhlmann}
Sara van~de Geer and Peter B{\"u}hlmann.
\newblock {$\ell_{0}$-penalized maximum likelihood for sparse directed acyclic graphs}.
\newblock \emph{Annals of Statistics}, 41\penalty0 (2):\penalty0 536 -- 567, 2013.

\bibitem[Verma and Pearl(1990)]{verma1990equivalence}
Thomas Verma and Judea Pearl.
\newblock Equivalence and synthesis of causal models.
\newblock In \emph{Uncertainty in Artificial Intelligence}, pages 255--270, 1990.

\bibitem[Xu et~al.(2024)Xu, Taeb, Küçükyavuz, and Shojaie]{xu2024integer}
Tong Xu, Armeen Taeb, Simge Küçükyavuz, and Ali Shojaie.
\newblock Integer programming for learning directed acyclic graphs from non-identifiable {Gaussian} models.
\newblock arXiv.2404.12592, 2024.

\bibitem[Ye et~al.(2020)Ye, Amini, and Zhou]{ye2020optimizing}
Qiaoling Ye, Arash~A Amini, and Qing Zhou.
\newblock Optimizing regularized {Cholesky} score for order-based learning of {Bayesian} networks.
\newblock \emph{IEEE Transactions on Pattern Analysis and Machine Intelligence}, 43\penalty0 (10):\penalty0 3555--3572, 2020.

\bibitem[Yu et~al.(2019)Yu, Chen, Gao, and Yu]{yu2019dag}
Yue Yu, Jie Chen, Tian Gao, and Mo~Yu.
\newblock {DAG-GNN: DAG} structure learning with graph neural networks.
\newblock In \emph{International Conference on Machine Learning}, pages 7154--7163. PMLR, 2019.

\bibitem[Zheng et~al.(2018)Zheng, Aragam, Ravikumar, and Xing]{zheng2018dags}
Xun Zheng, Bryon Aragam, Pradeep~K. Ravikumar, and Eric~P Xing.
\newblock {DAGs} with {NO TEARS}: continuous optimization for structure learning.
\newblock \emph{Advances in Neural Information Processing Systems}, 31:\penalty0 9492--9503, 2018.

\end{thebibliography}

\end{document}